\documentclass[twoside,11pt,jmlr]{article}
\usepackage{jmlr2e}

\usepackage{amsmath} 
\usepackage{amssymb} 
\usepackage[scr]{rsfso} 
\usepackage{graphicx}

\usepackage{rotating}

\usepackage{algorithm}
\usepackage{algpseudocode}

\jmlrheading{Under Review}{2019}{}{9/2019}{X/2019}{Merrill2019}{ZestFinance, Inc.}
\ShortHeadings{Generalized Integrated Gradients}{Merrill, Ward,
  Kamkar, Budzik and Merrill}

\begin{document}

\title{Generalized Integrated Gradients: A practical method for explaining diverse ensembles}

\author{\name John W. L. Merrill \email jwlm@zestfinance.com\\
\name Geoff M. Ward \email gmw@zestfinance.com \\
\name Sean J. Kamkar \email sjk@zestfinance.com \\
\name Jay Budzik \email j@zestfinance.com \\
\name Douglas C. Merrill \email douglas@zestfinance.com \\
\addr ZestFinance, Inc. \\
 3900 W Alameda Ave\\ 
Burbank, CA 91505
}

\editor{TBD}

\maketitle

\begin{abstract}%
  We introduce Generalized Integrated Gradients (GIG), a formal
  extension of the Integrated Gradients (IG)
  \citep{DBLP:journals/corr/SundararajanTY17} method for attributing
  credit to the input variables of a predictive model.  GIG improves
  IG by explaining a broader variety of functions that arise from
  practical applications of ML in domains like financial services.
  GIG is constructed to overcome limitations of Shapley \citeyearpar{shapley}
  and Aumann-Shapley \citeyearpar{aumann-shapley}, and has desirable
  properties when compared to other approaches.  We prove GIG is the
  only correct method, under a small set of reasonable axioms, for
  providing explanations for mixed-type models or games.
  We describe the implementation, and present results of experiments
  on several datasets and systems of models.
\end{abstract}%



\begin{section}{Introduction}\label{SectionIntroduction}

  Machine learning models are often described as ``black boxes'', since they often use many inputs and complex processing to generate a single score.  It is difficult to determine the reasons that support a given model-based decision directly from the model or system of models.  This limits the use of ML in high-stakes applications such as lending, health care, and education. Without an explanation of score differences among apparently-similar loan applicants, patients, or students, determining whether an ML model is behaving correctly and without bias is unknowable. As such, the application of machine learning in these domains has been limited.  

  Recent papers, notably SHAP \citep{lundberg2017unified} and IG \citep{DBLP:journals/corr/SundararajanTY17},  have proposed  methods to provide more transparency.  The method we describe builds on this work and contributes a new and more general method for assigning credit to variables used by an ML model to generate a score, and thereby provide the reasons why a given score was generated.

  
  As suggested by \citet{lundberg2017unified} and \citet{DBLP:journals/corr/SundararajanTY17}, one can view the ML model as a game: each feature is a player, the rules of the game are the ML model's scoring function, and the value of the game is the score given by the model. By representing the scoring process in this way, the problem of credit assignment can be reduced to the well-studied problem of credit allocation in a cooperative game.

  There are two axiomatically well-defined methods for allocating credit in a cooperative game: the {\it Shapley Value} \citep{shapley}, which applies to atomic games, in which there are discrete moves in the game, and the {\it Aumann-Shapley value} \citep{aumann-shapley}, which applies to infinitesimal games (games in which there are a continuum of moves, each one of which has an infinitesimal impact on the outcome of the game on its own, but in which a collection of moves makes a more noticeable difference).

  Allocating credit using Shapley values requires analyzing the output of the machine learning model in detail by computing the outcome of a ``lift'' of the model to a larger space. This is a problem: as we discuss below, there are many valid lifts of a given machine learning scoring function.  By contrast, allocating credit using Aumann-Shapley values requires computing a unique integral over the domain between two input vectors -- given a machine learning scoring function, the Aumann-Shapley values are unique.

  In this paper, we present Generalized Integrated Gradients (GIG), a credit assignment algorithm based on the Aumann-Shapley value that overcomes the limitations of both Shapley and Aumann-Shapley by applying the tools of measure theory.  The algorithm is a formal extension of IG that accurately allocates credit for a significantly broader class of games, including almost all of the scoring functions currently in use in the machine learning field, without making unrealistic assumptions about the data.

  GIG handles a broad class of predictive functions with both piecewise constant  (e.g. tree-based), continuous (e.g. neural network or radial basis function based), mixed models, including generalized location models  \citep{10.2307/2336722,Little:1986:SAM:21412}, and compositions thereof.  GIG is fully determined by its axioms, the predictive function and the data points under study -- with no free parameters or arbitrary choices required or even allowed.  In addition, GIG is the only method that can correctly assign credit for all composed functions of mixed type, a class which includes almost all machine learning functions currently in use. 
  
%

\end{section}


\begin{section}{Background}\label{SectionConceptualFramework}
  \begin{subsection}{Credit assignment techniques}\label{SubsectionCreditAssignmentTechniques}
    Throughout this paper, we shall use a definition of credit assignment functions which is general enough to encompass both Shapley and Aumann-Shapely values. We shall capture this notion of a credit assignment function in Definition \ref{DefinitionCreditAssignment}, which represents the difference in credit allocation between a single value $x$ and a set $S$: the ``amount'' the $k^{\rm th}$ variable contributes to moving from $x$ into $S$.

    For convenience, we use the term {\it function algebra} to refer to an algebra of functions ${\cal F}$ where each $f \in {\cal F}$ is of the form $f : {\mathbb R}^n \to {\mathbb R}$ imbued with the standard operations of point-wise addition and multiplication within the ring and the operation of multiplication by a given real. We state without proof that many of our results extend naturally to algebras of rings for which there is a well-defined Borel measure.

    \begin{definition}(A credit assignment function) \label{DefinitionCreditAssignment}
      Let ${\cal F}$ be a function algebra. A {\it credit assignment function} for ${\cal F}$ is a function $\Xi : {\cal F} \times {\mathbb R}^n \to {\mathbb R}^n$ such that:
      \begin{trivlist}
      \item
        {\bf Linear} If $f, g \in {\cal F}$, and $\alpha, \beta \in {\mathbb R}$, then $\Xi(\alpha f + \beta g, x) = \alpha \Xi(f, x) + \beta \Xi(g, x)$
      \item
        {\bf Insensitive to null variables} If $f \in {\cal F}$ is such that $f(x) = f(y)$ for any $x, y \in {\mathbb R}^n$ which differ only in the $i^{\it th}$ dimension, if the $i^{\it th}$ component of $\Xi(f, z) = 0$ for all $z \in {\mathbb R}^n$.
      \item
        {\bf Symmetric} For any $x \in {\mathbb R}^n$, let $\hat x_{ij}$ denote the result of interchanging the $i^{\it th}$ and $j^{\it th}$ columns in $x$. If $f \in {\cal F}$ is such that $f(x) = f(\hat x)$ for any $x \in {\mathbb R}^n$, then the $i^{\it th}$ and $j^{\it th}$ columns of $\Xi(f, z)$ are always equal.
      \item
      \end{trivlist}
    \end{definition}
    
    The symmetry axiom says that if two variables always have the same effect on the output of a function, then they must always have the same impact as measured by the credit assignment function.

    \begin{subsubsection}
      {Shapley Values}\label{SubsubsectionShapleyValues}

      \begin{definition}(Set function)\label{DefinitionSetFunction}
        A {\it real-valued set function}, $\nu$, is a function which maps the power set of a set to the reals.
      \end{definition}

      Throughout what follows, we shall focus only on set functions such that $\nu(\emptyset) = 0$. For convenience, we shall take the domain of $\nu$ to be some integer $N \geq 1$ since one can identify any set of input variables with $N$ by indexing them. 

      The Shapley values are functions of set functions obeying a set of axioms which appear very similar to the axioms which define credit allocation functions for a given ${\cal F}$. We shall describe how Shapley values relate to true credit allocation functions on ${\cal F}$ below.

      \begin{definition}(Shapley's Axioms)\label{DefinitionShapleyAxioms}
        Let $\nu:{\cal F} \times {\mathbb R}^n \times 2^n \to {\mathbb R}^n$ be a set function for each $f \in {\cal F}$. Then we say that $\phi : N \to {\mathbb R}$ {\it obeys Shapley's Axioms for ${\cal F}$} if $\phi : {\cal F} \times {\mathbb R}^n \to {\mathbb R}^n$ is
        \begin{trivlist}
        \item
          {\bf Efficient} For any $f \in {\cal F}$ and any $x \in {\mathbb R}^n$, $\sum_{i = 1}^n \phi_i(f, x) = f(x)$.
        \item
          {\bf Linear} If $f, g \in {\cal F}$, any $\alpha, \beta \in {\mathbb R}$ and any $x \in {\mathbb R}^n$,  $\phi(\alpha f + \beta g, x) = \alpha \phi(f, x) + \beta \phi(g, x)$.
        \item
          {\bf Insensitive to null variables} If $f \in {\cal F}$ and $x \in {\mathbb R}^n$ are such that $\nu(f, x, S) = \nu(f, x, S \cup \{ i \})$ for all $S \in ( N \setminus \{ i \} )$, then $\phi_i(f, x) = 0$.
            \item
              {\bf Symmetric} If $f \in {\cal F}$ and $i, j \in N$ are such that $\nu(f, x, S \cup \{ i \}) = \nu(f, x, S \cup \{ j \} )$ for all $S \subset N \setminus \{ i, j \}$, then the $\phi_i(f, x) = \phi_j(f, x)$ are equal for all $x \in {\mathbb R}^n$.
        \end{trivlist}
      \end{definition}

      \begin{definition}(Shapley values)\label{DefinitionShapleyValues}
        Let $\nu:{\cal F} \times {\mathbb R}^n \times 2^n \to {\mathbb R}^n$ be a set function. Then the {\it Shapley values} for $\nu$ at $f$ and $x$ are given by
        $$
        \phi_i(f, x) = \sum_{S \subseteq N \setminus \{ i \}} \frac{|S|! (N - |S| - 1)!}{N!} (\nu(f, x, S \cup \{ i \}) - \nu(f, x, S))
        $$
      \end{definition}

      The Shapley values have a very desirable property: they are unique.

      \begin{theorem}\label{TheoremShapleyValueUniqueness}(Uniqueness of Shapley)
        The Shapley Values are the unique values obeying the axioms in Definition \ref{DefinitionShapleyAxioms}  
      \end{theorem}

      The proof is given in \citet{shapley}.  Notice that Shapley values are computed based on elements of the power set of the input variables. This makes them computationally impractical to compute directly since there are exponentially many elements of the power set of a set.

      This is not the only limitation of the Shapley values as a means for assigning credit for a decision made by a machine learning scoring function. The Shapley values do not constitute a credit assignment function in the sense of Definition \ref{DefinitionCreditAssignment}. The computation of the Shapley values for any $f \in {\cal F}$ and $x \in {\mathbb R}^n$ requires that $f$ be defined on ${\mathbb R}^n \times 2^n$, but a credit assignment function is defined only for functions which belong to ${\cal F}$ -- that is, functions which are defined on ${\mathbb R}^n$. In order to impute the Shapley values corresponding to a given scoring function, one must lift the scoring function from ${\mathbb R}^n$ to the space ${\mathbb R}^n \times 2^n$. We refer to this function as a lift, since that is the standard terminology for such a function. (\citet{lundberg2017unified} refer to these functions as ``simplifications''.)

      The Shapley values computed depend on your choice of lift from the input space (which is dense and uncountable) into the discrete space required for the formula in Definition \ref{DefinitionShapleyValues}, and it is not obvious which of the many possible Shapley values (corresponding to the many possible lifts) are correct (including, e.g., \citet{lundberg2017unified,lundberg2018consistent}). That is, the Shapley values corresponding to a machine learning scoring function are not well-defined, as we now show.

      \begin{definition}(A lift of a function $f \in {\cal F}$ at a given point)\label{DefinitionLift}
        We shall say that $\nu$ is a lift of $f \in {\cal F}$ at $x \in {\mathbb R}^n$ to $2^n$ if $\nu(f, x, \emptyset) = 0$ and $\nu(f, x, \{ 1, 2, \ldots, n \}) = f(x)$.
      \end{definition}

      In the above, $2^n$ refers to a vector of binary values corresponding to the number of input variables $n$. For each $x \in {\mathbb R}^n$, the lift $\nu(f)$ associates a set of values, one for each subset for the set of input variables, with the single value $f(x)$, “lifting” $f$ from ${\mathbb R}^n$ to the larger space ${\mathbb R}^n \times 2^n$.
      
      \begin{definition}\label{DefinitionTrivialLift}(Two trivial lifts)
        Let $f \in {\cal F}$ and let $x \in {\mathbb R}$.
        \begin{itemize}
        \item
          The $N$-lift of $f$ at $x$ is the function
          $$
          \nu_0(f, x, X) =
          \begin{cases}
            f(x) & X = N \\
            0 & {\mathrm {o.w.}} \\
          \end{cases}
          $$
        \item
          The $\emptyset$-lift of $f$ at $x$ is the function
          $$
          \nu_N(f, x, X) =
          \begin{cases}
            0 & X = \emptyset \\
            f(x) & {\mathrm {o.w.}} \\
          \end{cases}
          $$
        \end{itemize}
      \end{definition}

      It would seem at first that these lifts would produce different values.  But the Shapley values for the $\emptyset$-lift and the $N$-lift of a given function at a given point are actually equal. The $i^{\it th}$  Shapley value associated with the $\emptyset$-lift is
      \begin{equation}\label{ShapleyEmptysetLift}
        \phi_i(f, x) = \sum_{S \subseteq ( N \setminus \{ i \} )} \frac{|S|!(N - |S| - 1)!}{N!} ( \nu(f, x, S \cup \{ i \}) - \nu(f, x, S))
      \end{equation}
      For all $S \subseteq (N \setminus \{ i \})$ except for $\emptyset$ itself, the difference in the last term of Equation \ref{ShapleyEmptysetLift} is zero. Thus we are left with
      \begin{eqnarray}
        \phi_i(f, x) & = & \frac{0! (N - 1)!}{N!} f(x) \\
        & = & \frac{f(x)}{N}
      \end{eqnarray}

      The complement to that argument shows the Shapley values corresponding to the $N$-lift are the same, except that the only set with a non-trivial contribution to the sum is the set $N \setminus \{ i \}$, and the order of the factorial terms in the remaining term are reversed.

      Consider a non-trivial lift function, however.

      \begin{definition}\label{DefinitionNonTrivalLift}(A non-trivial lift)
        Let $f \in {\cal F}$ and $x \in {\mathbb R}^n$. Let the {\it half-weight lift} of $f$ at $x$ be
        $$
        \nu_{\mathrm {HWL}}(f, x, X) =
        \begin{cases}
          f(x) & X = N \\
          f(x) - i & X = N \setminus \{ i \} \\
          0 & {\mathrm o.w.}
        \end{cases}
        $$
      \end{definition}

      In this case, an argument not parallel to the arguments above gives us
      $$
      \phi_i(f, x) = \frac{f(x) - i}{N}
      $$
      with a correction term
      $$
      \phi_0(f, x) = \frac{N + 1}{2}
      $$

      These values are different from the ones corresponding to the ones computed in Equation \ref{ShapleyEmptysetLift}. The point is, there is no unique set of Shapley values corresponding to any given machine learning function $f$.

      There is one special case in which there is a particular ``natural'' lift for a given machine learning function: if the function itself supports 'fill-in' for all columns. In many cases, machine learning systems must deal with absent values in their training sets, and, in those cases, it makes sense to view the function $\nu$ as a function $\nu : {\cal F} \times {\mathbb R}^n \times {\mathscr P}(N) \to {\mathbb R}$, defined by first taking
      $$
      \lambda(f, x, S) =
      \begin{cases}
        x_i & i \in S \\
        {\mathrm {NA}} & {\mathrm {o.w.}} \\
      \end{cases}
      $$
      where ${\mathrm {NA}}$ is a formal symbol denoting the notion that `the value is not available' and must be filled in by the model itself, and then taking
      $$
      \nu(f, x, S) = f(\lambda(f, x, S))
      $$
      We shall examine the relationship between this lift and several published applications of Shapley values to machine learning systems below.
    \end{subsubsection}

    \begin{subsubsection}{Differential credit allocation}\label{SubsectionDifferentialCredit}

      The Shapley values provide a local credit allocation mechanism, but their computation requires a lift from the machine learning function to a set function. In this subsection, we consider an alternative approach to credit allocation. Here, we do not attempt to find a single local credit allocation for a given input, but rather attempt to find a differential credit allocation which explains the reasons for the difference in scores between two separate inputs. For this problem, we consider a different kind of credit allocation function and a different set of axioms to which any such function must adhere.

      A {\it differential credit allocation function} is a function which explains the reasons that two outputs of a machine learning function differ.

      \begin{definition}(Differential credit assignment function)\label{DefinitionDifferentialCreditAxioms}
        A function $\mu : {\cal F} \times {\mathbb R}^n \times {\mathbb R}^n \to {\mathbb R}^n$ is a {\it differential credit allocation function} if and only if it is
        \begin{trivlist}
        \item
          {\bf Efficient} For any $f \in {\cal F}$ and any $s, e$ in ${\mathbb R}^n$, $\sum_{i = 1}^n \mu_i(f, s, e) = f(s) - f(e)$.
        \item
          {\bf Linear} For any $f, g \in {\cal F}$, any $x, y \in {\mathbb R}^n$, and any $\alpha, \beta \in {\mathbb R}$, $\mu(\alpha f + \beta g, s, e) = \alpha \mu(f, s, e) + \beta \mu(g, s, e)$
        \item
          {\bf Insensitive to null variables} For any $f \in {\cal F}$, and any $i \leq n$, if $f(x)$ is constant along the $i^{\it th}$ component of its domain, then for any $s, e \in {\mathbb R}^n$, $\mu_i(f, s, e) = 0$
        \item
          {\bf Symmetric} if $f \in {\cal F}$ and $i, j \leq n$ are such that for any $x \in {\mathbb R}^n$, $f$ is unchanged when the $i^{\it th}$ and $j^{\it th}$ components of $x$ are interchanged, then the $i^{\it th}$ and $j^{\it th}$ components of $\mu(f, s, e)$ are always equal.
        \end{trivlist}
      \end{definition}

      Notice how similar the axioms in Definition \ref{DefinitionDifferentialCreditAxioms} appear to those in Definition \ref{DefinitionShapleyAxioms}, but notice also that a function which obeys Definition \ref{DefinitionDifferentialCreditAxioms} is a scoring function in the sense of Definition \ref{DefinitionCreditAssignment}. For any given lift, there is a straightforward relationship between the Shapley values associated with a given function and the definition of a credit allocation function as given above in Definition \ref{DefinitionCreditAssignment}:

      \begin{definition}(Differential Shapley)\label{DefinitionDifferentialShapley}
        Assume we have a lift of a scoring function $f \in {\cal F}$ into a set function $\nu_{f, x}$ for all $x \in {\mathrm R}^n$ as in Definition \ref{DefinitionLift}. Let $\phi_f(x)$ denote the Shapley values $\nu_{f, x}$ for each $x \in {\mathrm R}^n$. We can then define
        $$
        \Xi(f, x, S) = \phi_f(x) - E(\phi_f(y) | y \in S)
        $$
        for any $S \in \Sigma$.
      \end{definition}

      It is clear from the axioms characterizing the Shapley values, Differential Shapley is a credit assignment function in the sense of Definition \ref{DefinitionCreditAssignment}.
      
      The critical difference is that the axioms in Definition \ref{DefinitionDifferentialCreditAxioms} reflect behavior in an abstract extension of the domain of the machine learning function at a given point, and the axioms in Definition \ref{DefinitionShapleyAxioms} are effectively combinatorial, and reflect behavior in an abstract lift of the machine learning function to a higher dimensional space.

      In this subsection, we shall focus on a particular class of differential credit assignment functions: path-integral based differential credit assignment functions which may be applied to explain machine learning functions like neural networks, which are everywhere infinitely differentiable.

      \begin{theorem}\label{IGUniquenessTheorem}
        The only path integral based differential credit allocation function for continuously differentiable functions is IG.
      \end{theorem}

      The proof is given in  \citet{aumann-shapley}.  Sundararajan et al. \citeyearpar{DBLP:journals/corr/SundararajanTY17} show that symmetry reduces the set of non-self-intersecting paths upon which this computation can be performed to the single linear path between the endpoints.

      \begin{definition}(IG)\label{DefinitionIG}
        Let $f \in {\cal F}$ be everywhere continuously partially differentiable.   Then for any $s, e \in {\mathbb R}^n$, let 
  \begin{equation}\label{IG}
    \Xi_{IG}(f, s, e) =
    (e - s) \int_0^1 \nabla f((1 - \alpha) s + \alpha e) d \alpha
  \end{equation}
  
      \end{definition}

      From the axioms in Definition \ref{DefinitionDifferentialCreditAxioms} and Theorem \ref{IGUniquenessTheorem}, we see this is the mechanism for computing the Aumann-Shapley values in the sense of Definition \ref{DefinitionCreditAssignment}.
    \end{subsubsection}

  \end{subsection}

  \begin{subsection}{Comparing margin space to transformed spaces}\label{SubsectionTransformation}
    Because our primary interest in this work is in loan underwriting in the United States in 2019, where there are specific requirements to explain each model-based decision, it is often required to apply a smoothed ECDF of the output of a machine learning function. In this subsection, we discuss the general question of transformed output and explain the reasoning behind this particular choice.  For example, consider the following problems:
    
    \begin{trivlist}
    \item
      {\bf Loan application approval} Typically, lenders want to approve a fixed percentage of loan applications or wish to approve a set of loan applications which consume a given line of credit with optimal ROI, or some combination of both. The first of these is performed in rank space, and the second in what one would term ``expected payoff'' space.
    \item
      {\bf Medical outcome prediction} Machine learning can be used to predict whether a given medical procedure will be successful.  A physician wishes to understand the probability of success.
    \end{trivlist}

    In each of these cases, the natural interpretation of credit allocation is not in the output space of the machine learning function, but rather in the inverse transform of that output into the space in which the target function is defined. In the case of ranked loan application approval in which one wishes to approve a fixed fraction of all loan applications, for instance, one is not interested in the output of the machine learning system but rather in the rank of each loan within the collection of all loans. In the case of optimal ROI, one is not interested in the output of the machine learning system, but rather in the estimated ROI for any given loan. In the case of medical outcome prediction, one is not only interested in the probability of success, but also in the best way or ways to make the treatment or intervention more likely to succeed.

    Each of these cases requires predictions be transformed into a space other than the pure margin space output of the machine learning scoring function.  We refer to this transformed space as ``score space''.  To be useful in many applications, a credit assignment function must solve the more difficult problem of assigning credit in score space.


    Since we do not know the distribution of scores in margin space, we instead use a smoothed empirical cumulative distribution function (ECDF,  \citet{dodge2006oxford}) to approximate the output. We shall use a piecewise linear spline approximating the ECDF to guarantee almost-everywhere differentiability and also computational efficiency. We shall refer to this as a ``Smoothed ECDF'' below.
  \end{subsection}
\end{section}

\begin{section}{Some difficult credit assignment problems}\label{SectionSampleScoringFunctions}
  In this section, we present several sample machine learning scoring functions that arise in practical applications, and discuss issues prior approaches have explaining them.  
  
  Consider the two scoring functions shown in Figure \ref{base_pair}. Both functions use a gradient boosting machine to predict an outcome.  GBM's produce piecewise constant functions. In the panel on the left, the machine learning model output is used directly. In the panel on the right, the output of that machine learning model is passed through a smoothed ECDF to transform the scores from the original non-uniform distribution to a uniform distribution more appropriate for making, e.g., loan underwriting decisions. As discussed above, we refer to the former of these output spaces as {\it margin space}; the latter {\it rank space}. 

IG cannot be applied to either of these, as the core machine learning score is piecewise constant and therefore not differentiable. Lundberg, et al. \citeyearpar{lundberg2018consistent} describe TreeExplainer for efficiently computing a Shapley value for the left-hand system. However, the method described in \citet{lundberg2018consistent} proposed for mapping explanations into score space relies on the assumption of statistical independence of the model inputs, a property which is almost never true in practical applications.

  \begin{figure}
    \centering
    \includegraphics[width=0.7\textwidth]{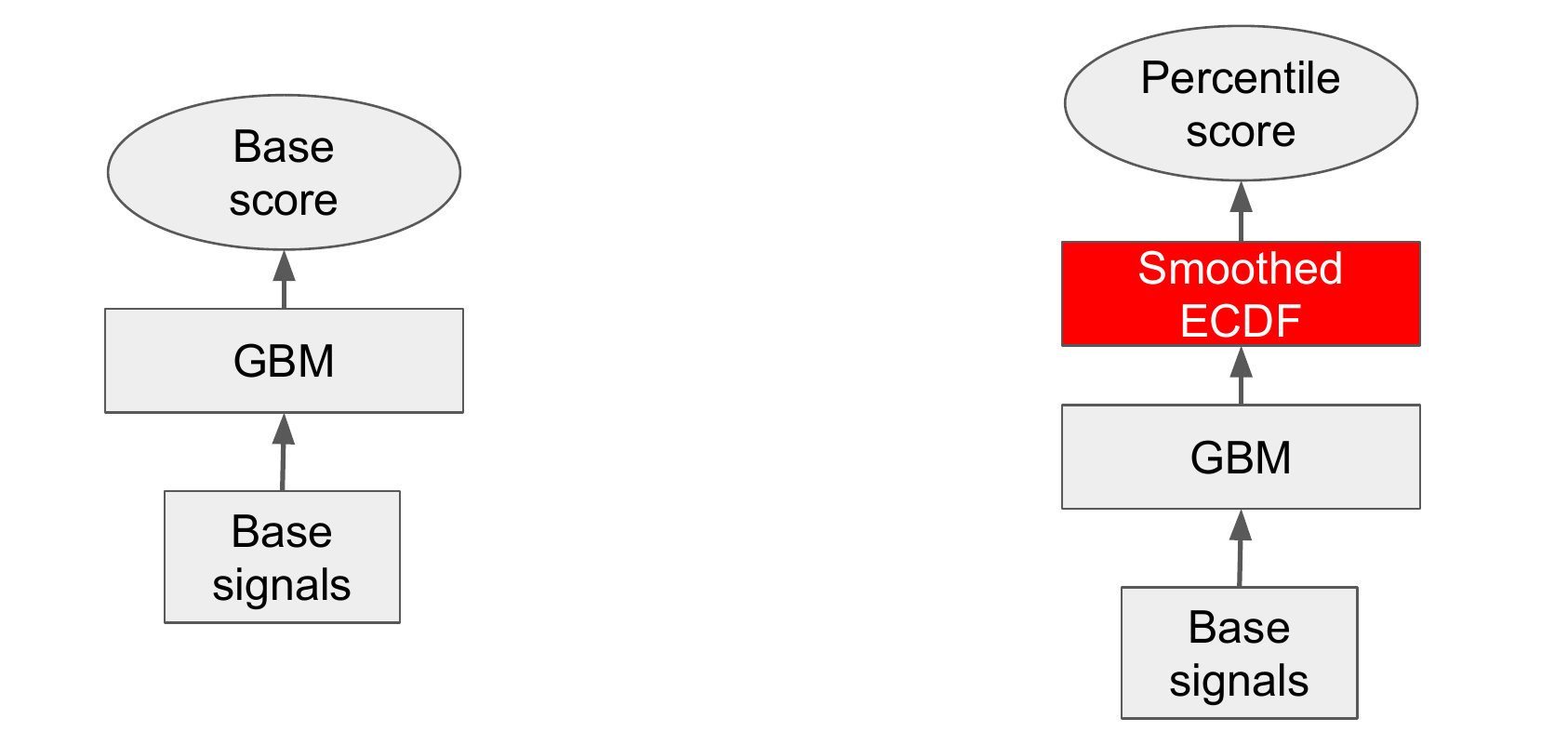}
    \begin{caption}
      {Two machine learning systems for credit assignment. GIG, TreeExplainer, and Shapley can assign credit to the system on the left, but only GIG and Shapley can assign credit to the system on the right.}    \label{base_pair}
    \end{caption}
  \end{figure}

  Next, consider the case of a linear ensemble of mixed model types that is a linear combination of scores from a GBM (piecewise constant), and neural network (continuous).  While credit can be assigned using prior techniques for this type of model,\footnote{First, use IG to allocate credit for the neural network and SHAP to allocate credit for the GBM; then apply the linear ensemble function to the credit assignments to arrive at the final credit assignment for the ensemble. \citet{MixedCreditPatent}} this same task becomes much more difficult when the results are passed through a smoothed ECDF as is required for their use in a decisioning system.

  
  Finally, consider a non-linear ensemble function which combines many submodels of mixed type. Such non-linear ensemble methods have been used to improve the accuracy of an ensemble further than could be done with a linear combination of scores \citep{NonlinearEnsemble,WolpertStack}.  Moreover these more complex ensembles can be used to add additional constraints such as fairness to the ensemble score.

  We show below that only GIG (first described in \citealp{GIGPatent}) can provide axiomatically sound differential credit assignment for compositions of piecewise constant and continuous functions such as those described above.  Such function compositions commonly arise in domains like finance, where predictive accuracy is of utmost importance and a high-stakes decision is being made that must be explained.

  
\end{section}

\begin{section}{Intuitive description of the GIG method}\label{SectionIntuition}
  Let us start by considering IG, the jumping-off point for GIG.

  As discussed above, IG is an instance of Aumann-Shapley: it assigns credit between a pair of inputs, to determine the ``reasons why'' the model's score changed between those two inputs. Given the output of a continuous machine learning model (like an ANN), IG assigns credit by computing the component-wise integral of the gradient of the output of the model on the path from one input to the other as in Definition \ref{DefinitionIG}.  This is equivalent to the Aumann-Shapley value.

  \begin{figure}
    \centering
    \includegraphics[width=1.0\textwidth]{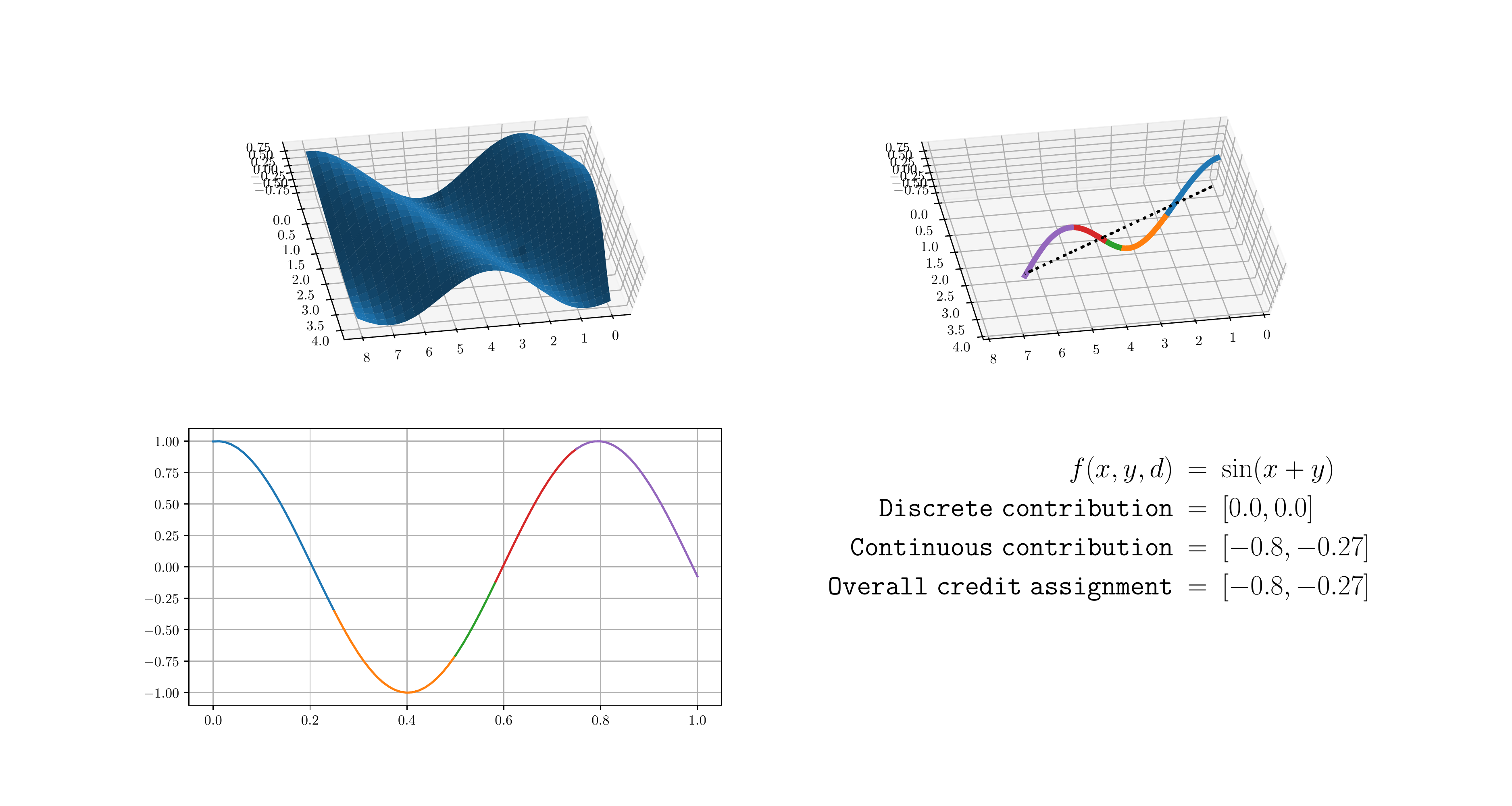}
    \begin{caption}
      {The interpretation of IG (and GIG) for a smooth function $f$ on ${\mathbb R}^2$ and $d$. The upper-left panel is $f$.  The upper-right panel, $f$ with respect to a path from $(0.5, 1)$ to $(7,3)$.  The lower-left panel, the cross section on the hyperplane induced by $f$ and $p$.  The lower-right is the credit assignment.  In this case, the discrete function $d$ maps all its inputs to 0, and therefore has no contribution. }\label{sine_only}
    \end{caption}
  \end{figure}

  To see what this means in pictures, see Figure~\ref{sine_only}. Here, $f$ is $\sin (x + y)$. The top left portion of the figure shows the graph of the function as a surface in 3-D, the top right, the path upon which the IG values will be evaluated, the bottom left, the function as restricted to the path, and the bottom right, the amount of credit assigned.

  If we allocate credit for the changes between the points $(0.5, 1)$ and $(7, 3)$, we would simply integrate the partial derivatives along the straight line between those two points, yielding a total credit allocation of $(-0.8, -0.27)$.

  \begin{subsection}{Extending IG to piecewise constant functions}\label{SubsectionGigPiecewiseConstant}
    
    In order to motivate the construction which follows, let us consider the problem of extending this kind of process to handle a piecewise constant function of a very specific kind: a function where there is a single boundary that is perpendicular to a single axis.

    We can not hope to integrate the partial derivatives of that discontinuous function, since the partials are almost everywhere zero, and undefined wherever they are not zero. Instead, one can mimic the process by which the Dirac delta function is constructed: build a sequence of continuous approximations to the discontinuous function, such that the limit is equal to the discontinuous function \citep{dirac}. One could then perform the integral in Equation \ref{IG} for each of those approximates and let the limit of those credit allocations serve as the credit allocation for the original discontinuous function. The resulting credit assignment would correctly allocate the size of the step as the amount of credit assigned, and would assign it, unavoidably, to the single variable upon which the function was defined.

    Unfortunately, this solution only works in functions of one variable. It assigns the appropriate amount of credit, but gives no indication of how to allocate that credit among many variables. We can, however, apply the technique in a special case: the case of a simple linear boundary which is orthogonal to one of the axes of the  function's inputs. For this case, we can construct a sequence of approximates as above, which all approach the discontinuity the same way, by making the value of each of the functions constant along each hyperplane parallel to the boundary, and approximate a step function with ever greater accuracy. We can then look at the limit of the integrals along the path which passes through those functions. This yields the expected result: at the limit, all of the credit is allocated along the axis along which the discontinuity occurred.  This process is shown in Figure~\ref{discrete}.

      \begin{figure}
      \centering
      \includegraphics[width=1.0\textwidth]{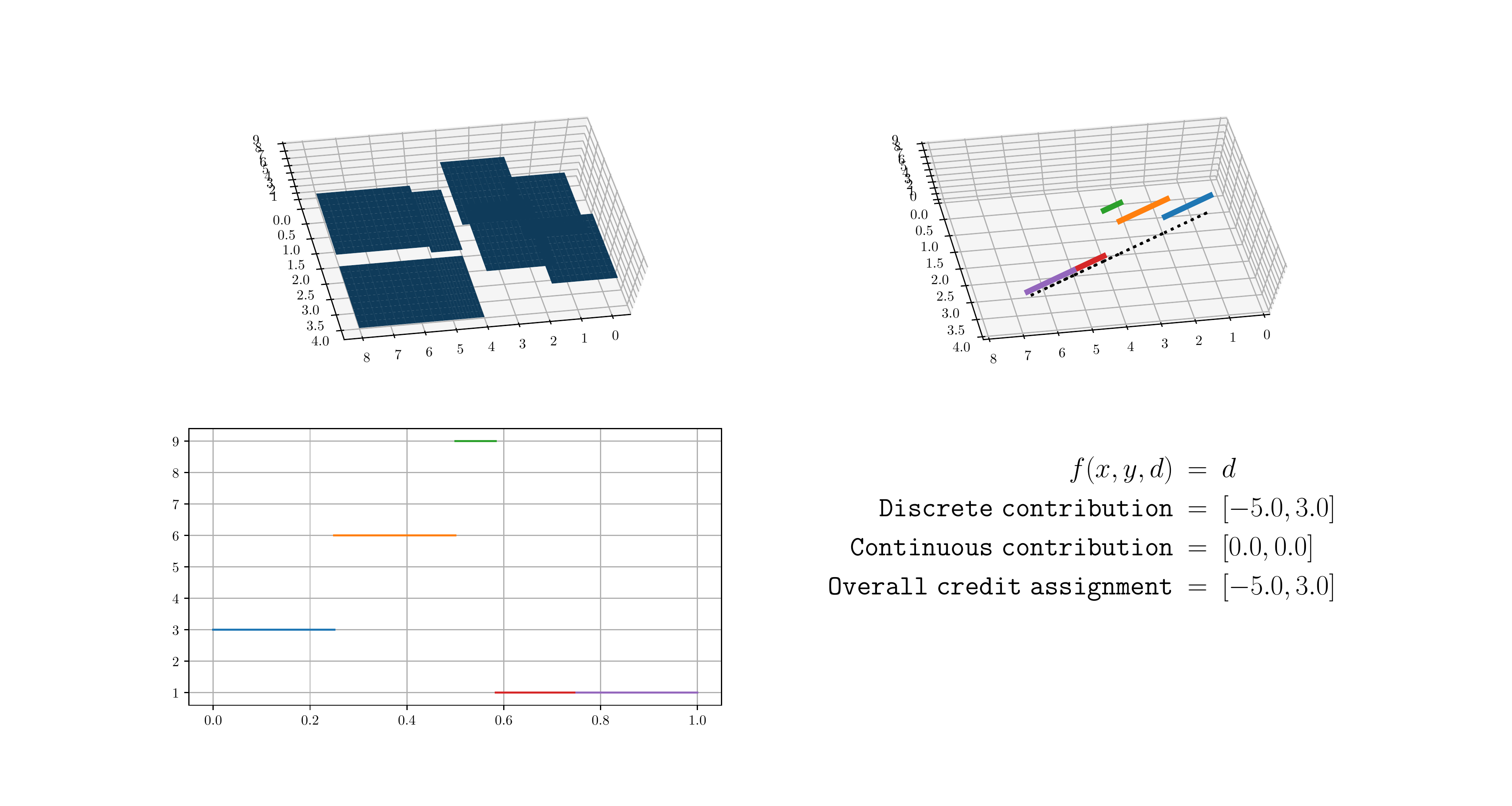}
      \begin{caption}
        {The credit allocated for a piecewise constant function $d$ on ${\mathbb R}^2$.  The panels are as in Figure \ref{sine_only} and the path the same.  We walk the path and measure the jumps in $x$ and $y$.  Here, $f$ has no contribution because it maps all its inputs to $d$. }\label{discrete}
      \end{caption}
    \end{figure}

    \begin{figure}
      \centering
      \includegraphics[width=0.5\textwidth]{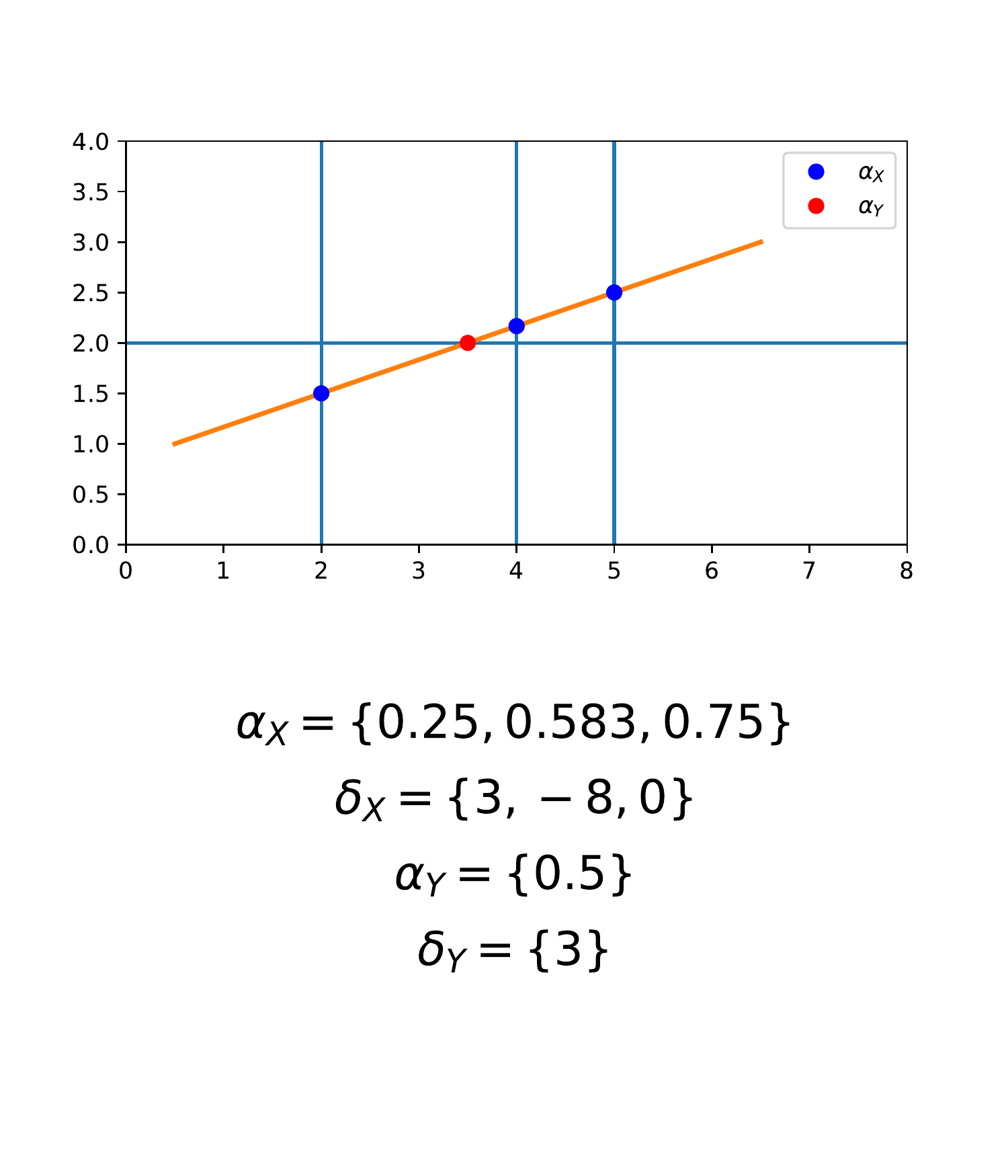}
      \begin{caption}
        {A schematic of a piecewise constant function with discontinuities perpendicular to the axes. The value of the function within each rectilinear cell is shown in the middle of that cell.}\label{discrete_grid}
      \end{caption}
    \end{figure}

    Observe that the actual path by which we traverse the discontinuity is irrelevant; the pattern of credit allocated is always exactly the same: all credit is allocated perpendicular to the boundary no matter what path is followed. In the case of IG in general, credit allocation varies with the path, but in this case, it does not.

    However, this mechanism is ill-defined in general. We can make it work if we chose the sequence of functions that are constant parallel to the hyperplanes upon which the discontinuity occurs. But there are many other sequences of functions, and different sequences yield different credit allocations. In the sections below, we formalize this process and prove the sequence of functions described above is the only solution which conforms to the axioms of interest.

    Examples of the GIG method are shown in several figures. Figure~\ref{discrete_grid} shows the values of a piecewise constant function of two variable with discontinuities perpendicular to the axes. Figure~\ref{discrete} shows how credit would be assigned for a linear path within that space for that particular function.

    We can extend this process to handle functions with more than one discontinuity, at least as long as the discontinuities are hyperplanes and the path along which the integral is made never intersects more than one discontinuity simultaneously. That case follows immediately from the proof of the single hyperplane case discussed above: the credit is assigned along the normal to the plane. That result can be extended to the case of smooth boundaries for which the normal is defined everywhere and the path intersects no more than one boundary at any given point -- one simply assigns credit along the normal at each intersection.

    The proof below deals with hyperplanes orthogonal to one or more axes where the path intersects one or more hyperplanes simultaneously. It, however, does not extend to the case of hyperplanes not orthogonal to a given axis, to say nothing of the general case of smooth boundaries.
  \end{subsection}

  \begin{subsection}{Extending IG to Piecewise Continuous Functions}\label{SubsectionGigPiecewiseContinuous}
    We can further extend the technique outlined in subsection \ref{SubsectionGigPiecewiseConstant} to a broad class of piecewise continuous functions by making a simple observation.

    Assume that $f$ is a piecewise continuous function whose discontinuities are always perpendicular to a single axis. We can consider decomposing the path along which we would compute the IG integral into a set of contiguous pieces each of which contained only a single discontinuity. If we could figure out how to allocate credit within each of those pieces, we could allocate credit to the entire path by adding the pieces together.

    We solve this problem by decomposing the piecewise continuous function $f$ into two subfunctions, a continuous function $f_C$ which captures the continuous aspect of $f$ and a piecewise constant function $f_D$ which captures the discontinuous aspect of $f$ such that $f = f_C + f_D$. That decomposition is obvious: letting $\hat \alpha$ be the unique value such that $f$ is discontinuous at $(1 - \hat \alpha) a + \hat \alpha b$. We then define
    $$
    f_D(\alpha) =
    \begin{cases}
      \lim_{\alpha \to \hat \alpha^-} f((1 - \alpha)a + \alpha b) & \alpha < \hat \alpha \\
      \lim_{\alpha \to \hat \alpha^+} f((1 - \alpha)a + \alpha b) & \alpha > \hat \alpha \\
    \end{cases}
    $$
    and we let $f_C = f - f_D$. It is straightforward to see that $f_C$ and $f_D$ meet the criteria above, and so the computation behaves as desired.

    In general, however, the problem of computing the two one-sided limits of the value of $f$ would at first seem to be impractical. However, with a small change of formalism and an observation, this problem goes away.

    The key observation is that any function $f$ on ${\mathbb R}^n$ which is piecewise differentiable off a set of hyperplanes which are perpendicular to one of another axis can be written in the form $g(x, d(x))$ for $x \in {\mathbb R}^n$ where $g$ is everywhere differentiable and each element of $d$ is piecewise constant off a set of hyperplanes each of which is perpendicular to a single axis. One performs this computation by enumerating the cells within which the $d_i$ are constant, obtaining a finite set of cells $C_1, C_2, \ldots, C_k$, and creating a function $d$ defined as
    $$
    d(x) =
    \begin{cases}
      1 & x \in C_1 \\
      2 & x \in C_2 \\
      \vdots & \vdots \\
      k & x \in C_k
    \end{cases}
    $$

    We can then extend $f$ to an everywhere differentiable function $g(x, n)$ by using a set of sufficiently smooth spline bases with compact and disjoint support.

    For any function of this form, we can compute the one-sided limits as follows:
    \begin{eqnarray}
      \lim_{\alpha \to \hat \alpha^-} f((1 - \alpha)a + \alpha b) =
                                                                  & g(((1 - \alpha)a + \alpha b), d(\frac{((1 - \alpha)a + \alpha b)}{2})) \\
      \lim_{\alpha \to \hat \alpha^+} f((1 - \alpha)a + \alpha b) =
                                                                  & g(((1 - \alpha)a + \alpha b), d(\frac{1 - ((1 - \alpha)a + \alpha b)}{2}))
    \end{eqnarray}

    An example showing the computation of the one-sided limits for the piecewise constant function is shown in Figure \ref{DiscreteAlphaCurve}.    Examples showing how this method can then be used to assign credit for two piecewise continuous functions are shown in Figures \ref{sine_and_offset} and \ref{compound_merger}.

    \begin{figure}
      \centering
      \includegraphics[width=0.5\textwidth]{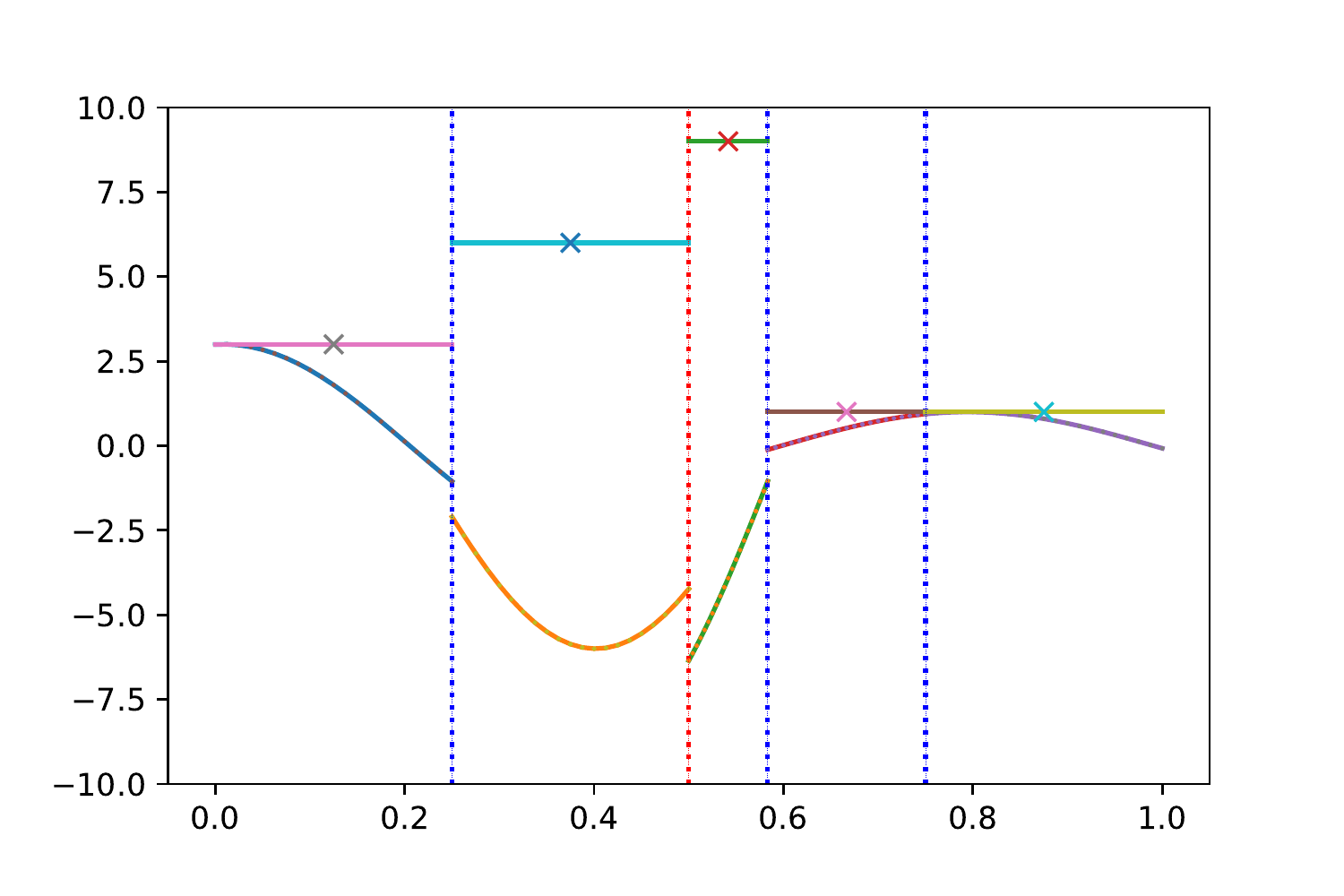}
      \begin{caption}
        {A schematic of how the individual one-sided limits are computed, by taking the midpoints of the segments along which $d$ is constant.}\label{DiscreteAlphaCurve}
      \end{caption}
    \end{figure}

    \begin{figure}
      \centering
      \includegraphics[width=1.0\textwidth]{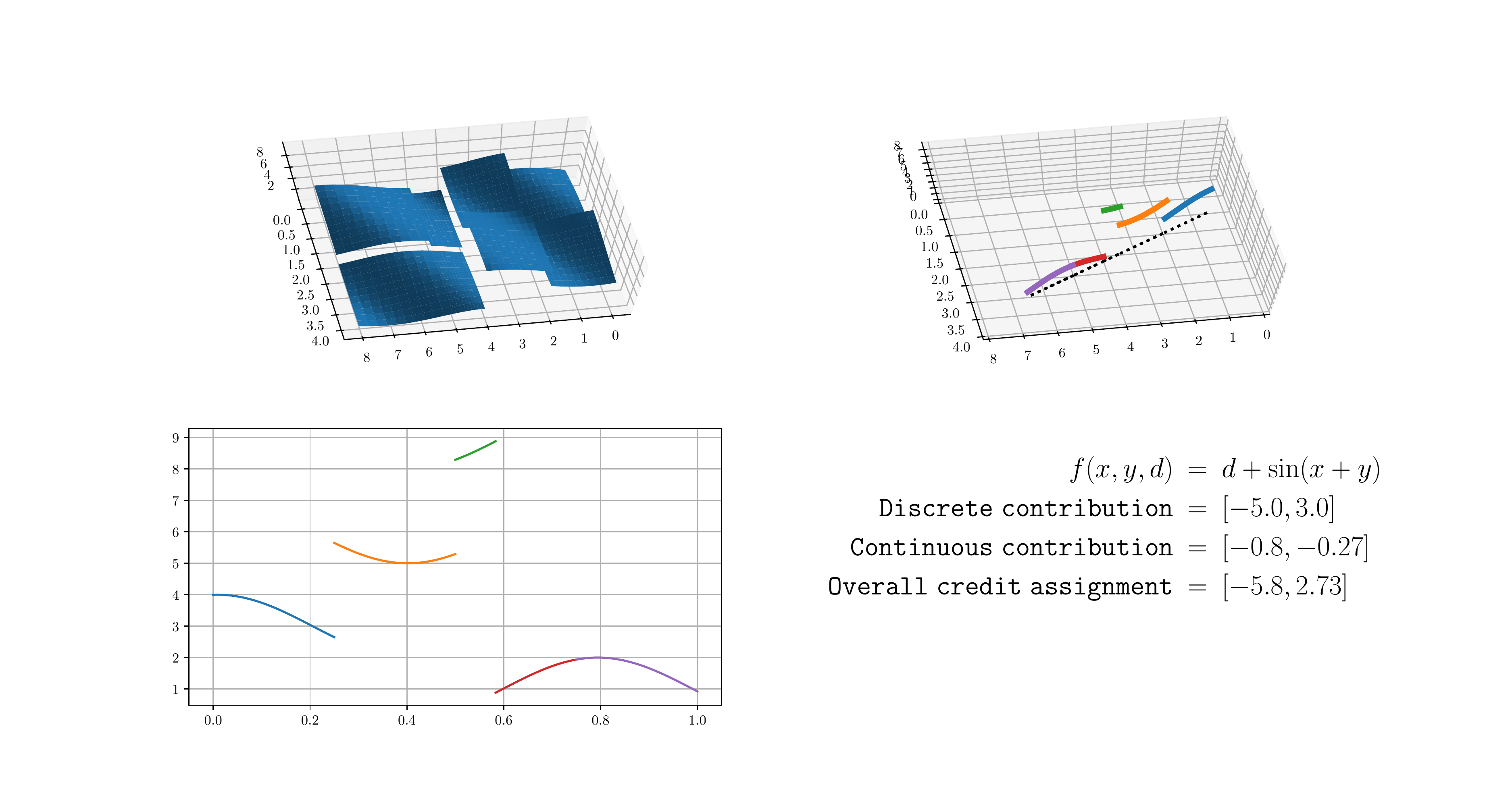}
      \begin{caption}
        {A straightforward example of credit assignment for a function which is the sum of a sine function and the piecewise continuous function displayed in Figure \ref{discrete}.}\label{sine_and_offset}
      \end{caption}
    \end{figure}

    \begin{figure}
      \centering
      \includegraphics[width=1.0\textwidth]{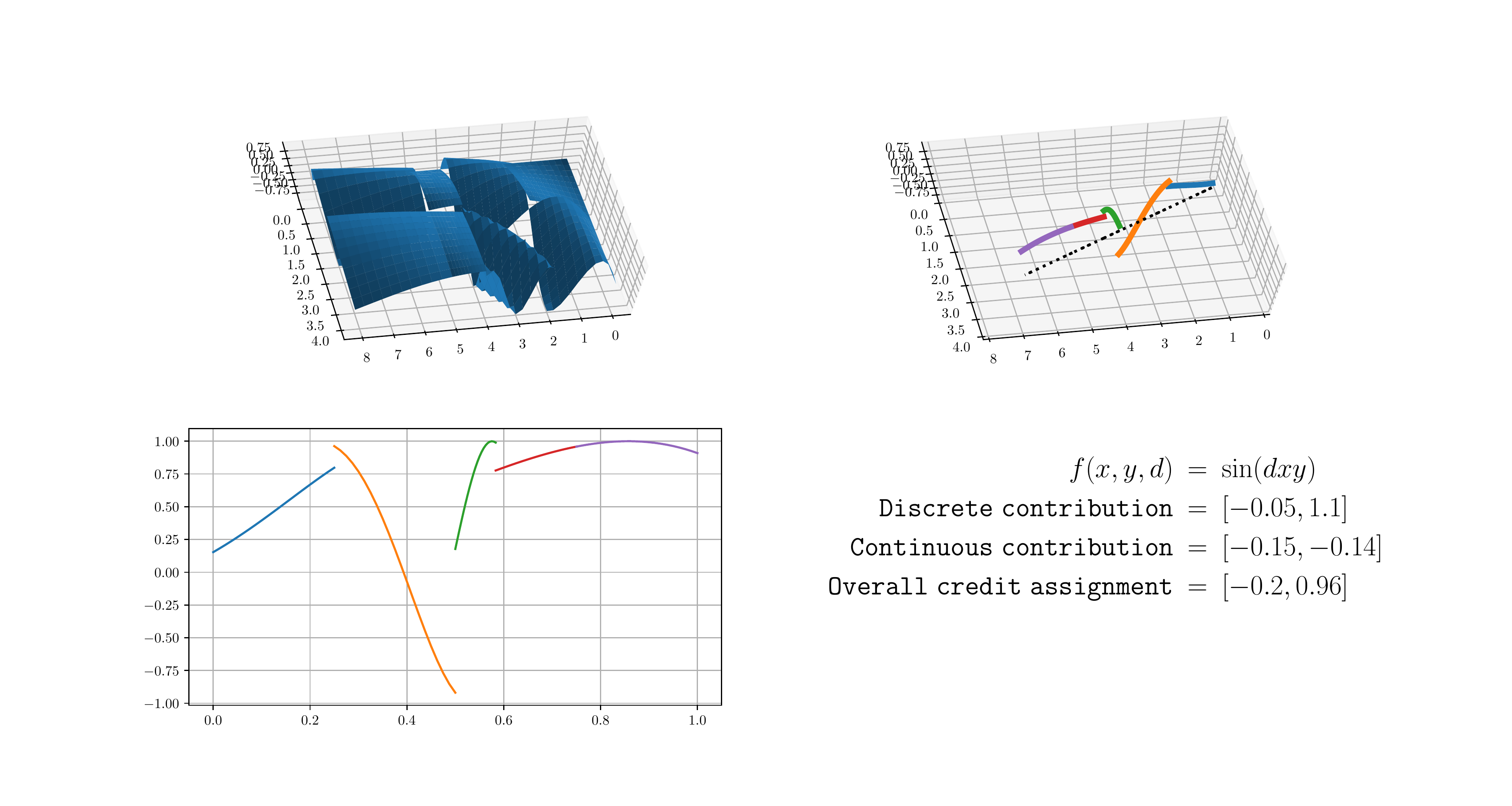}
      \begin{caption}
        {A more complicated example of credit assignment for a function which is the sine of a sine function applied to a complex product of the values along the two axes and the piecewise constant function shown in Figure \ref{discrete}.}\label{compound_merger}
      \end{caption}
    \end{figure}

  \end{subsection}
\end{section}

\begin{section}{Generalized Integrated Gradients}
  We start with two definitions.  First, the standard definition of an {\it orthant}.

  \begin{definition}\label{DefinitionOrthant}(Orthant)
        Let $x \in {\mathbb R}^n$. Then the {\it orthant containing $x$} is the set of all points $y \in {\mathbb R}^n$ each component of which all have the same sign as the corresponding component of $x$.  Given any $d \in \{ -1, 1 \}^n$, the {\it orthant indexed by $d$} is the orthant containing the point $d$.
  \end{definition}

  An orthant is the $n$-dimensional generalization of the same concept as quadrants: an orthant is the set of points around the origin separated by the orthogonal planes through the origin. The one-dimensional orthants are the positive and negative reals; the two-dimensional orthants are the quadrants; the three-dimensional orthants are the eight octants; and so on. The one-dimensional orthant indexed by $(1)$ is the set of positive reals, and the corresponding orthant indexed by $(-1)$ is the set of negative reals.

  We need a definition corresponding to the class of models we are studying (trees, neural networks, etc., and compositions thereof) that is rigorous enough to permit formal analysis:  We shall say that a function $f : {\mathbb R}^n \to {\mathbb R}$ is {\it piecewise continuous off a set of orthogonal hyperplanes} if and only if there exists a collection of collections, one collection for each dimension of the input space, $ \{ \{ b_{ij} : 1 \leq j \leq m_i \} : 1 \leq i \leq n \} $ such that $f$ is continuous on the line segment between $s \in {\mathbb R}^n$ and $e \in {\mathbb R}^n$ unless there exists an $i \leq n$ and a $j \leq m_i$ such that $B_{ij}$ lies between $s_i$ and $e_i$. See Figure \ref{compound_merger}  for a visual example.  More formally:

  \begin{definition}\label{PCOH}(Piecewise continuity off a set of orthogonal hyperplanes)
    A function $f : {\mathbb R}^n \to {\mathbb R}$ is piecewise continuous off a set of orthogonal hyperplanes if and only if
    \begin{itemize}
    \item
      There exists a function $D : {\mathbb R}^n \to {\mathbb R}^k$ which is constant everywhere except along a set of individual hyperplanes. That is, there exists a set of values $\{ b_{ij}  : 1 \leq i \leq n, 0 \leq j \leq j_i \}$ where we take $b_{i0} = -\infty$ and $b_{i{j_i}} = \infty$ by convention such that $D$ is constant on each Cartesian product $\prod_{i = 1}^n (b_{i{k_i}}, b_{i{(k_i + 1)}})$ for a sequence $\{ k_i: 0 \leq k < j_i \}$.
    \item
      There exists a continuous function $g : {\mathbb R}^n \times {\mathbb R}^k \to {\mathbb R}$ such that $f(x) = g(x, D(x))$.
    \end{itemize}
  \end{definition}
  \begin{example}\label{PCoOHpix} (Examples of functions which are piecewise continuous off a set of orthogonal hyperplanes.)
    \begin{itemize}
    \item
      Any continuous function is piecewise continuous off a set of orthogonal hyperplanes. This includes not just neural network models, but also linear models, SVMs, radial basis function models and a variety of other well-known model classes.
    \item
      The output of a model based on decision trees is piecewise continuous off a set of orthogonal hyperplanes. These models include decision trees, random forests, boosted tree models such as a GBM, and forests of extra-random trees (ETFs).
    \item
      The output of a continuous ensemble of any of the above. That is, the property of being piecewise constant of a set of orthogonal hyperplanes is preserved under continuous transformation, such as when a set of diverse submodels is composed with a neural network in a system of models.
    \item
      Any of the above passed through a differentiable function such as a smoothed ECDF, which is often employed to make machine learning scores more useful for decision-making.
    \end{itemize}
  \end{example}

  There is no computationally feasible mechanism in the general case of computing the integral required for GIG for an arbitrary path. Here we focus on describing a practical mechanism for implementing GIG in the case of a linear path between two points. Since this is the same path as is used in IG, the generalization of IG to GIG is straightforward.
  
  The implementation of GIG is also straightforward. First, given a set of boundaries (which can be readily produced via depth-first search of the decision trees) and a linear path (defined by the endpoints $s$ and $e$ in a differential credit assignment query), we enumerate the points where the path intersects one or more of the boundaries. Then, the sets of hyperplanes which intersect the path at each point are assembled. Next, we decompose the function along the path into two functions, one continuous along the path, and the other piecewise constant off of the intersection of the line segment with the set of discontinuities. We then compute the contribution of the continuous part of the decomposition and the contribution of the piecewise constant part of the decomposition. We add those two parts; the result is the credit assigned by GIG.

  Furthermore, we make a uniqueness claim about GIG: it is the only credit assignment process which satisfies several axioms. We use the terminology from \citet{shapley} and \citet{aumann-shapley} where appropriate. Since GIG is an extension of IG, we extend the axioms which define IG with three new axioms:

  \begin{definition}\label{GigAxioms}(Extended Aumann-Shapley axioms)
    A differential credit allocation function satisfies the {\it extended Aumann-Shapley axioms} if it is a differential credit assignment function as in Definition \ref{DefinitionDifferentialCreditAxioms} and also satisfies the following four axioms:
    \begin{trivlist}
    \item
      {\bf Strongly symmetric} A credit function $\Xi$ is strongly symmetric if and only if for any $f, g \in {\cal F}$ and any $i, j < n$, if $f(x_1, x_2, \ldots, x_i, \ldots, x_j, \ldots, x_n) = g(x_1, x_2, \ldots, x_j, \ldots, x_i, \ldots, x_n)$, then the $i^{\mathrm {th}}$ component of $\Xi(f, s, e)$ is the same as the $j^{\mathrm {th}}$ component of $\Xi(g, s, e)$ and vice versa as long as the function $f$ is everywhere continuous between $s$ and $e$.
    \item
      {\bf Insensitive to remote change} A credit assignment function $\Xi$ is insensitive to remote change if and only if for any function $f \in {\cal F}$ and any path, $p$, there exists an open set $O$ with $p((0, 1)) \subset O$ such that for any $g \in {\cal F}$ and any open $P \subseteq O$, if $g$ agrees with $f$ for all points within $P$, then $\Xi(f, p) = \Xi(g, p)$.
    \item
      {\bf Insensitive to constant variables} A credit assignment function $\Xi$ is insensitive to constant variables if and only if for any $f \in {\cal F}$ and any $s, e \in {\mathbb R}^n$ where $s_i = e_i$ for some $i \leq n$, the $i^{\mathrm {th}}$ component of $\Xi(f, s, e)$ is $0$.
    \item
      {\bf Reflexivity} A credit assignment function $\Xi$ is reflexive if for any $f \in {\cal F}$ and any $x, y \in {\mathbb R}^n$, $\Xi(f, x, y) = -\Xi(f, y, x)$.
    \end{trivlist}
  \end{definition}

  In order to construct any credit allocation function which supports all functions which are piecewise continuous off a set of orthogonal hyperplanes and which extends IG, we must show that IG meets all the extended Aumann-Shapley axioms. In the next few paragraphs, we give an intuitive explanation of each of the new axioms and explain why IG satisfies each of them.

  Strong symmetry is the extension of symmetry to a broader class of functions. Where symmetry applies only to an inversion which takes a function to itself, strong symmetry relates the credit assigned between two functions which are related by flipping variables. Effectively, this says the if the input to two functions is related by flipping two variables, the relation is symmetric. The proof that IG is strongly symmetric is simply a restatement of the proof that IG is symmetric, with one minor change which arises from the relationship between the partial derivatives. In IG, continuity on the line segment between the two points follows from the fact that the input functions must be everywhere continuous.

  Locality simply means that we don't need to know the value of the scoring function on any point outside the path in order to compute the credit allocation for that path. One can view this as an extension of insensitivity to null variables: a null variable doesn't change the function anywhere. Insensitivity to remote change is its local equivalent.

  Insensitivity to an invariant variable simply reflects the fact that if something didn't change then there is no credit to assign to it: every bit of the change in outcome happened without any change in the component, so the component didn't do anything. This is trivially true for IG as the partial derivative corresponding to an unchanged variable will always be 0.

  
  This allows us to formally define Generalized Integrated Gradients and state the main result of this paper.

  \begin{theorem}\label{GigTheorem}(Generalized Integrated Gradients, GIG)
    Let $f : {\mathbb R}^n \to {\mathbb R}$ be piecewise continuous off a set of orthogonal hyperplanes and let $s, e \in {\mathbb R}^n$. Let $\alpha_j$ denote the values of $\alpha$ such that $(1 - \alpha) s + \alpha e$ lies on a boundary. Then there is a unique set of values, $\zeta_j \in {\mathbb R}^n$ and a set of values $\iota_0, \iota_1 \in {\mathbb R}^n$ such that
    \begin{equation}\label{EqnGigFormula}
    \zeta + \iota_0 + \iota_1 + \sum \int_{\alpha_j}^{\alpha_{j + 1}} \nabla f(x) dx
    \end{equation}
    satisfies extended Aumann-Shapley axioms in Definition \ref{GigAxioms}. 
  \end{theorem}

  The proof is in several parts. First, we establish some basic properties which any realization of the $\zeta$ and the $\iota$ values would need to have. Then we work with piecewise constant functions, showing:
  \begin{enumerate}
  \item  
    The contribution of any given discontinuity lies only in the dimensions across which the function is discontinuous. That is, if there are two variables $x_i$ and $x_j$ which are discontinuous at a given point, then the only non-zero components of $\zeta$ corresponding to that point are the $i^{\mathrm {th}}$ or the $j^{\mathrm {th}}$. (Lemma \ref{GigRestrictVariables}.)
  \item
    We may consider each individual discontinuity in isolation: the values of $\zeta$ or $\iota$ depend only on the discontinuity to which each one corresponds. (Lemma \ref{GigSingleDiscontinuity}).
  \item
    The contribution of a single discontinuity is computed by taking a single sum based on a set of known coefficients which depend only on the number of dimensions across which the function is discontinuous at that point. (Lemmas \ref{GigOrthantSum}, \ref{GigOrderOnly}, and \ref{GigConstantDecomposition})
  \end{enumerate}

  We then explicitly consider the case where the path passes through a discontinuity, performing the following steps:
  \begin{enumerate}
  \item
    Formally define the mechanism by which one computes $\zeta$ from $\eta_O$ for a single discontinuity of radix $k$.(Lemmas \ref{GigOrthantSum} and \ref{GigOrthantDecomposition}.)
  \item
    Describe how to compute the $\eta_O$ value for the characteristic function of a specific orthant $O$ for an intersection of radix $k$. (Lemma \ref{GigOrderOnly} and Theorem \ref{GigConstantDecomposition}.)
  \item
    Formally describe how the $\zeta$ value for that discontinuity is computed (Definition \ref{GigZetaConstruction}) and then show how the $\iota$ values at each endpoint at computed (Definition \ref{GigIotaConstruction}). 
  \item
    Prove the resulting credit allocation function satisfies the extended Aumann-Shapley axioms. (Theorem \ref{GigZetaIotaSatisfaction})
  \item
    Show the sum of IG on the continuous portions of $f$ and the values of $\zeta$ and $\iota$ is the only function that satisfies the extended Aumann-Shapley axioms.
  \end{enumerate}

  \begin{subsection}{Isolating the discontinuities}\label{GigIsolateDiscontinuities}

  We start with a simple restatement of insensitivity to constant variables:

  \begin{lemma}\label{GigRestrictVariables}
    For $f \in {\cal F}$, if $s$ and $e$ lie on $H$, a single orthogonal hyperplane of dimension $l < k$ and, then $\Xi(f, s, e)$ also lies within $H$.
  \end{lemma}

  We also present a simple invariance case:

  \begin{lemma}\label{GigSinglePoint}
    Let $f$ be everywhere constant off a single point $z \in {\mathbb R}^n$, and let $s, e \in {\mathbb R}$ be distinct. Then $\Xi(f, s, e) = 0$.
  \end{lemma}

  \begin{proof}
    Without loss of generality, assume that $z = 0$. Observe that $f$ now satisfies the hypothesis of the Symmetry axiom for all pairs $s_i, e_j$, and that we therefore know that $\Xi(f, p)_i = \Xi(f, p)_j$ for all $i$ and $j$. But we also know, by Linearity, that $\sum_i \Xi(f, p)_i = 0$. If the sum of a set of equal values is zero, then each value is 0. 
  \end{proof}
  
  \begin{lemma}\label{GigSingleDiscontinuity}{\bf (Single intersections are universal)}
    Given any $f \in {\cal F}$, any path $p$ from $s, e \in {\mathbb R}^n$, and any set of hyperplanes, we may reduce the problem of computing $C(f, s, e)$ to the sum of its values along a partition of $p$ into a set of open intervals and two isolated points (the endpoints of the path) such that each of $\Xi$ along the full path is $t$, for each of which there is only one value of $\alpha$ at which the sub-path intersects with a hyperplane and a set of subsets at the boundaries of these hyperplanes.
  \end{lemma}

  \begin{proof}
   Pick a partition of the path, a finite number of disjoint convex open segments, such that the union of their closures covers the path, and such that any given open set contains at most one discontinuity. There are now three types of items which cover the path: the finitely many open segments, the finite collection of isolated interior points which lie in the intersections of their closures, and the two endpoints. By Lemma \ref{GigSinglePoint}, we know that the interior points don't enter into the final computation, and hence we are reduced to the three cases listed in Lemma \ref{GigSingleDiscontinuity}.
  \end{proof}

  Notice that we can immediately decompose the computation of the assignment under GIG into a continuous portion consisting of the continuous segments of each of the convex open sets in the decomposition used in the proof of Lemma \ref{GigSingleDiscontinuity}, and a discontinuous portion consisting of the isolated discontinuities along the path, including the endpoints. Since the open sets in the partition are disjoint, and cover the path, we can immediately compute their contribution to the final values by integrating across the continuous portions. That means that we need only figure out how to handle each discontinuity taken in isolation to compute the final credit assignments arising from GIG.

\end{subsection}

  \begin{subsection}{Assigning Credit for a single discontinuity}\label{GigIsolatedDiscontinuity}

  We now turn to the problem of assigning credit for a function which is piecewise constant off a set of hyperplanes each of which is orthogonal to a single axis corresponding to a path which passes through the intersection of those hyperplanes. Without loss of generality, we may assume that the intersection lies at the origin. Furthermore, by Lemma \ref{GigRestrictVariables}, we may assume that the $f: {\mathbb R}^k \to {\mathbb R}$, where $k$ is the radix of the intersection -- that is, the number of hyperplanes in the set.

  By reflection of variables from positive to negative (and, presumably, reversing the signs of the amounts of credit assigned), we may further assume that the path proceeds from the purely negative orthant $\eta_{(-1, -1, \ldots, -1)}$ to the purely positive orthant $\eta_{(1, 1, \ldots, 1)}$.

  \begin{definition}\label{GigOrthantDecomposition}
    Let $f \in {\cal F}$ be piecewise constant off the hyperplanes defining this segment. For each orthant let
    $$
    \eta_{d}(f, x) =
    \begin{cases}
      f(d) & x \in \eta_{d} \\
      0 & o.w.
    \end{cases}
    $$
  \end{definition}

  The decomposition described in Definition \ref{GigOrthantDecomposition} is used to prove Proposition \ref{GigOrthantSum}, which plays a critical role in the proof below.
  
  \begin{proposition}\label{GigOrthantSum}
    Let $f \in {\cal F}$ be as in Definition \ref{GigOrthantDecomposition}, and let $C$ be any credit assignment function satisfying the conditions laid out in Definition \ref{GigAxioms}. Let $x \in {\mathbb R}^k$ have no zero components. Then
    $$
    \Xi(f, -x, x) =
    \sum_{{d} \in \{ -1, 1 \}^k}  \eta_{d}(f, x)
    $$
  \end{proposition}
  \begin{proof}
    By Lemma \ref{GigRestrictVariables}, we know that the values on the boundaries of the orthants which lie between two points contribute nothing to the final credit assignments. The proposition then follows directly from additivity.
  \end{proof}

  The case where $s$ and $e$ lie on a single hyperplane of lower dimension than $k$ is complicated: instead of being ignorable, the credit assignment must lie within that hyperplane. This proof will also be delayed until later.

  \begin{proposition}\label{GigOrderOnly}
    Let $d \in \{ -1, 1 \}^k$ and let $d_i = ( 1, 1, 1, \dots, 1, -1, -1, \ldots -1) \in \{ -1, 1 \}^k$ be such that $d_i$ has $i$ $1$'s and $(k - i )$ $-1$'s, where $d$ also has $i$ $1$'s and $(k - i)$ $-1$'s. Let $\pi$ be any permutation which takes the elements of $d$ onto those of $d_1$. Then for any $x \in \eta_{d}$ and any $f \in {\cal F}$, $\eta_{d}(f, x) = \eta_{d_i}(f(\pi(x), \pi(x))$.
  \end{proposition}

  Proposition \ref{GigOrderOnly} follows immediately from the definition of a credit assignment function.

  Henceforth, we shall focus our attention on the special case of the function $X$, the characteristic function of $\eta_{(1, 1, \ldots 1)}$.
  
  \begin{theorem}\label{GigConstantDecomposition}
    Let $d \in \{ -1, 1 \}^k$. Then there is an $\eta_{d_i, k} \in {\mathbb R}^k$ such that $\eta_{d}(X, x) = \eta_{d_i}(X, d)$ for any  $x \in {\mathbb R}^k$ with no zero components.
  \end{theorem}
  \begin{proof}
    We prove this in two pieces. First we prove that it holds for $d_I = (1, 1, \ldots, 1)$ for all $k$. Then, for each given $k$, we show that the constant value for $d_I$ yields a constant value for any other $d$.

    First, the case of $d_I$. Observe that every pair of variables $x_i$ and $x_j$ enter symmetrically in $d_I$ and the transformation takes the orthant onto itself, whence we know by symmetry that the $i^{\mathrm {th}}$ and  $j^{\mathrm {th}}$ components of $\eta_d(X, x)$ must be equal. Furthermore, we know that the sum must be $1$, from which we get
    $$
    \eta_{d_I}(X, x) = \biggl (
    \frac{1}{k}, \frac{1}{k}, \cdots, \frac{1}{k}
    \biggr )
    $$

    Now, let us consider the case of some other $d$. For our purposes, we only need to look at $d$ which are of the form $(1, 1, \ldots, 1, -1, -1, \ldots, -1)$ in which the positive and negative components of the indicator are in blocks.  We proceed by induction on the number of negative elements of the indicator, as follows.

    Let $p$ denote the number of negative elements in the indicator. The case of $p = 0$ is taken care of above. Now consider the case of $p = 1$. Suppose we take the characteristic function of orthant $d_{(1, 1, 1, \ldots, 1, -1)}$ and adjoin characteristic function of the orthant $d_I$ aside it. The resulting combined function is the characteristic function of $d_J$ where $J$ is the indicator $(1, 1, \ldots 1)$ for the corresponding orthant of dimension $k - 1$. By linearity, we know that $\eta_d(X, x) + \eta_{d_I}(X, x) = \eta_{d_J}(X, x)$, and since $\eta_{d_I}$ and $\eta_{d_J}$ are both independent of $x$, then $\eta_d(X, x)$ must also be independent of $x$.

    We can extend this construction to larger values of $p$. For $p = 2$, for instance, we would adjoin the values of $\eta_{d_I}(X, x)$, $\eta_{d_{(1, 1, \ldots, 1, -1)}}$, and $\eta_{d_{(1, 1, \ldots, -1, 1)}}$, observing that each one is constant on the relevant orthant, and then notice that we had just eliminated one more plane of discontinuity. From this, we would have the proof that the value of $\eta_{d_{(1, 1, \ldots, 1, -1, -1)}}(X, x)$ must be independent of $x$. For $p = 3$, we would adjoin one quadrant with $p = 0$ -- that is $d_I$ -- three quadrants with $p = 1$, and three quadrants with $p = 2$, and then observe that the remaining orthant would then complete a block which had eliminated a second hyperplane.

      Iterating this, let $\eta_{(k, p)}$ denote the value of $O$ in the case where the indicator contains $p$ negative values for all $p < j$, we obtain
        \begin{equation}
          \eta_{(k - j, 0)} = \eta_{(k, j)} + \sum_{p = 0}^{j - 1} \binom{j}{p} \eta_{(k, p)}
        \end{equation}
        from which, by simple rearrangement, we obtain
        \begin{equation}\label{GigZetaValues}
          \eta_{(k, j)} = \eta_{(k - j, 0)} - \sum_{p = 0}^{j - 1} \binom{j}{p} \eta_{(k, p)}
        \end{equation}
        Which is independent of $x$, as desired.
  \end{proof}

  The inductive computation of $\eta_{(k, j)}$ can be simplified. One can show that
  \begin{align}
  \eta_{(k, j)} & = \frac{1}{k} \binom{(k - 1)}{j}^{-1} \nonumber \\
  & =  \frac{j! (k - j - 1)!}{k!} \label{GigEtaShapley}
  \end{align}
  which are the coefficients for the corresponding terms in the computation of the Shapley values. We shall discuss this  relationship in Subsection \ref{SubsectionShapleyRelationship}, below.
  
\end{subsection}

  \begin{subsection}{Construction of $\zeta$}\label{GigZetaConstruction}

  In this section, we describe how the items for the interior discontinuities and the boundary discontinuities are computed. We start with the construction of the value of $\zeta$ for an arbitrary function at an arbitrary discontinuity approached from an arbitrary direction.  The following definitions guarantee that all of the extended Aumann-Shapley axioms hold.

  \begin{definition}\label{GigConstruction}(The value of $\zeta$ corresponding to a discontinuity)
    Let $f \in {\cal F}$ be piecewise constant off a set of $k$ orthogonal hyperplanes where the intersection of the hyperplanes is at the origin and let $d \in \{ -1, 1 \}^k$ represent the orthant from which the path begins (so that $-d$ represents the orthant into which the path then proceeds.) Let $\tau$ be the natural inversion which takes $d$ to $(-1, -1, \ldots, -1)$.  Let
    $$
    S = \{ -1, 1 \}^k \setminus \{ (-1, -1, \ldots, -1), (1, 1, \ldots 1) \}
    $$
    and 
    $$
    \Lambda(f, d) = \tau^{-1} \biggl (\frac{f(-d) - f(d)}{k} +
    \sum_{e \in S} \eta_{\tau(e)} f(e) \biggr )
    $$
  \end{definition}

  \begin{definition}\label{GitZetaValue}
    Let $f \in {\cal F}$ be piecewise constant off a set of orthogonal hyperplanes which intersect at the origin, and let $x \in {\mathbb R}^k$. Let $d \in \{ -1, 1 \}^k$ be the index of the orthant from which $x$ is drawn. Let $\zeta = \Lambda(f, d)$.
  \end{definition}

  \begin{theorem}\label{GigZetaValue}
    Then the function
    $$
    \Xi(f, x, -x) = \zeta
    $$
    satisfies the extended Aumann-Shapley axioms if $\zeta$ is as in Definition \ref{GitZetaValue}.
  \end{theorem}
  \begin{proof}
    It is straightforward to see that $\Xi$ is linear, efficient, and insensitive to null variables. It is equally clear that it is insensitive to remote change. It is insensitive perpendicular to an axis of symmetry by construction of the values of $\eta$ -- if $\alpha f + \beta g$ eliminates a variable (which is equivalent to being symmetric across that variable), then that variable no longer enters into the computation of the values of $\Lambda$, and so the result is insensitive to that axis. Strong symmetry follows directly from the construction of the values of $\eta$.
  \end{proof}

  We have now constructed $\zeta$ for items in the middle of the path.
\end{subsection}
\begin{subsection}{Computing $\iota_0$ and $\iota_1$}\label{GigIotaConstruction}
  Construction of $\iota_0$ and $\iota_1$ is a little more complicated.

  Consider the extension of $f$ to the orthants outside the point of intersection as if the point of intersection were in the middle of the interval instead of an edge. To begin, assume there is no extra discontinuity at the actual intersection beyond that implicit in the function $f$ as viewed along the edge. Observe that we now have something which looks like the sum of entering from the left and leaving from the right, so the sum of entering from the left ($\iota_1$) and leaving from the right ($\iota_0$) must be $\zeta$. Since we know that the result of adding these together gives us $\zeta$, we see we must have two artificial functions $f_0$ and $f_1$ which look like the original function except that they are $0$ to the left in the case of $\iota_0$ and $0$ to the right in the case of $\iota_1$. The other orthants must each be multiplied by some $\alpha_d$ for $\iota_0$ and $1 - \alpha_d$ for $\iota_1$. By strong symmetry -- symmetry is not enough here -- we obtain that there is some sequence $\alpha_n$ such that $\alpha_d = \alpha_n$ for all $d$ with $n$ positive values in its indicator set.

  We aren't quite done yet, though. There is the possibility that the value at the endpoint is not the same as the limit taken along the path from the right or from the left (in fact, this will generally be true). So we need to account for that value. To do this, we consider the singleton along the hyperplane in which we are working. By invariance of the function (that is, the function's value is completely independent of any transformation which takes the plane onto itself), we know that the contribution of that singleton is equal along all components of the hyperplane. Then we add the contribution of that singleton (in the positive sense to $\iota_o$ and the negative sense to $\iota_1$. We now have efficiency.

  Finally, we need to determine the value of $\alpha_d$ as stated above.

  Let us consider a piecewise constant function $f \in {\cal F}$  with a single discontinuity at $e \in {\mathbb R}^n$, and let $s \in {\mathbb R}^n$ be anywhere. Consider the path from $s$ to $e$ and its reciprocal path from $e$ to $s$. By reflexivity, we know that $\Xi(f, s, e) = -\Xi(f, e, s)$. We also know that the weight on any orthant $O_d$ on the $s$ to $e$ path is $1 - \alpha_d$ and that the corresponding weight on the $e$ to $s$ path is $\alpha_d$, but on the function reversed. (Note that this takes full reflexivity -- efficiency is not enough.) Taking all this together, we obtain $1 - \alpha_d + (- \alpha_d) = 0$, so $\alpha_d = \frac{1}{2}$ for all $d$ other than the entering or leaving orthants.
\end{subsection}

\begin{theorem}\label{GigZetaIotaSatisfaction}
  $\zeta$, $\iota_0$ and $\iota_1$ yield a credit allocation function which satisfies the extended Aumann-Shapley axioms.
\end{theorem}
\begin{proof}
  \begin{trivlist}
  \item
    {\bf Efficient} The contribution of the orthants 'along the axis' is the difference of the values of the function along that axis.
  \item
    {\bf Linear} The construction of $\zeta$ from the values of $\eta_{n, j}$ is linear.
  \item
    {\bf Insensitive to null variables} Values are only assigned to coefficients upon which there are discontinuities, and no null variable contributes a discontinuity, so their contribution is always $0$.
  \item
    {\bf Strongly symmetric} Since the values of $\eta$ depend only upon the number of $1$'s and $-1$'s in the signature for each orthant, strong symmetry follows immediately.
  \item
    {\bf Insensitive to remote change} The contributions of each discontinuity depend only on an arbitrarily small open ball around the discontinuity, hence are insensitive to remote changes.
  \item
    {\bf Insensitive to constant variables} Since the values of $\eta$ are only assigned relative to the hyperplanes upon which there is variation, the resulting function is insensitive to constant variables.
  \item
    {\bf Reflexivity} This is true in the midpoint case by construction and in the endpoint case by selection of $\alpha_d$.
  \end{trivlist}
\end{proof}

  \begin{subsection}{GIG values for compound continuous-discrete functions}\label{GigCompound}

  \begin{definition}\label{GigFormalDecomposition}
      Let $f \in {\cal F}$ be piecewise continuous off a set of orthogonal hyperplanes and let $x, y \in {\mathbb R}^k$ be such that there exists at most one $\alpha \in [0, 1]$ such that $\phi(\alpha) = f((1 - \alpha) x + \alpha y$ is not continuous at $\alpha$. Let $\phi_D$ and $\phi_C$ be such that both $\lim_{\epsilon \to \alpha^-}\phi_C(\epsilon, d)$ and $\lim_{x \to \alpha^+} \phi_C(\epsilon, d)$ both exist and zero for each intersection, and $\phi_D(\epsilon, d) = \lim_{\epsilon \to 0^+} f((1 - \alpha) x + \alpha y + \epsilon d$ be such that $\lim_{\epsilon \to 0^+} f((1 - \alpha) x + \alpha y + \epsilon d) = \lim_{\epsilon \to 0^+} \phi_C(\epsilon, d) + \lim_{\epsilon \to 0^+} \phi_D(\epsilon, d)$.
  \end{definition}
Finally:

  \begin{proposition}\label{GigWellDefined}
    Let $f$ be piecewise continuous off a set of orthogonal hyperplanes. Then for any $x, y \in {\mathbb R}^n$, be the result of taking the construction as in Definition \ref{GigFormalDecomposition} at each discontinuity on the straight line path from $x$ to $y$. This construction is unique and well-defined.
  \end{proposition}
  
  The proof is trivial.  
  
  \begin{theorem}\label{GigDecomposition}
    Let $\iota_0$, $\iota_1$, and $\zeta$ be computed from $\phi_D$ as defined in Proposition \ref{GigWellDefined} for each intersection in turn.  Let
    $$
    \Xi(f, x, y) =  \sum \zeta_j + \iota_0 + \iota_1 + \sum \int_{\alpha_j}^{\alpha_{j + 1}} \nabla f(x) dx
    $$

    $\Xi$, so defined, satisfies the extended Aumann-Shapley axioms.
  \end{theorem}

\begin{proof}  The proof of Theorem \ref{GigDecomposition} follows directly from the fact that IG satisfies the extended Aumann-Shapley axioms, that the values of $\iota_0$, $\iota_1$, and $\zeta$ for the individual discontinuities satisfy the Aumann-Shapley axioms, and that the weighted sum of any such pair also satisfies the extended Aumann-Shapley axioms.

  This shows the existence half of Theorem \ref{GigTheorem}. Uniqueness follows directly from linearity and the uniqueness of the values of $\iota_0$, $\iota_1$, and $\zeta$ and from the uniqueness of IG.

  This concludes the proof of Theorem \ref{GigTheorem}.
\end{proof}
\end{subsection}

\end{section}

\begin{section}{Details of the implementation of GIG}\label{GigDefinition}
  Let $d : {\mathbb R}^n \to {\mathbb R}^k$ denote the aggregation of the underlying piecewise constant functions. Let $f : {\mathbb R}^n \times {\mathbb R}^k \to {\mathbb R}$ denote the underlying continuous function. Let $D_1, D_2, \ldots, D_n$ denote the set of possible hyperplanes on which $d$ might be discontinuous. As normal, let $\nabla f : {\mathbb R}^n \times {\mathbb R}^m \to {\mathbb R}^n = \{ \frac{\partial f(x, d)}{\partial x_1}, \ldots, \frac{\partial f(x, d)}{\partial x_n} \}$ where the partial derivatives are all defined. Then GIG is naturally implemented as in Algorithm \ref{GigAlgorithm}.

  The use of the set of all possible discontinuities along the path -- even if many of them are trivial -- is intentional. The internal constant $\delta$ is taken to be such a small step that a perturbation of any of the intersections by an amount of $\delta$ in the direction of any orthant remains within the same continuous cell as the corner at which the difference is to be computed lies. In order to guarantee this, we need to take all possible cells into consideration, even those for which there's no discontinuity along the path -- the perturbation off the axis of the path might otherwise step across a boundary.

  In the algorithm description, we use a shorthand when we multiply $\delta$ by $D$. That multiplication only takes place on the axes along which the discontinuity takes place; all other values of the point are unchanged.

  \begin{algorithm}\label{GigAlgorithm}
    \caption{Generalized Integrated Gradients}
    \begin{algorithmic}[1]
      \Procedure{GIG}{}
      \Require $f$, $s$, $e$, $D_1, D_2, \ldots, D_n$
      \Ensure $\Xi(f, s, e)$
      \State $A \gets \{\alpha_i \in (0, 1) : (1 - \alpha) s + \alpha e \in \cup_{i = 0}^k D_i\}$ \Comment{Possible internal discontinuities}
      \State $\epsilon \gets \frac{1}{2} \min(\{\frac{\alpha_1}{2}, \frac{\alpha_2 - \alpha_1}{2}, \ldots, \frac{\alpha_{|A|} - \alpha_{|A| - 1}}{2}, \frac{1 - \alpha_{|A|}}{2})$
      \State $\delta\gets (1 - \epsilon) s + \epsilon e$
      \State $R \gets \{| \{ 1 \leq j \leq n : (1 - \alpha_i) s + \alpha_i e \in D_j \} | : 1 \leq i \leq |A| \}$ \Comment{Radices of each discontinuity}
      \For{$i \in 1, 2, \ldots, |A|$}\Comment{Compute contribution of each internal discontinuity}
      \State $O_i \gets (1 - \alpha_i) s + \alpha e$
      \For{$D \in \{ -1, 1 \}^{R_i}$}
      \State $\zeta_{i, D}\gets \eta_D f(O_i, d(O_i + \delta D))$ \Comment Uses $D$ as a vector
      \EndFor
      \State $\zeta_i \gets \sum_D \zeta_{i, D}$
      \EndFor
      \State $\zeta \gets \sum_{i = 1}^{|A|} \zeta_i$
      \State $R^0\gets | \{1 \leq j \leq n : s \in D_j \} |$\Comment{Radix at starting point}
      \If{$R^0 > 0$}
      \For{$D \in \{ -1, 1 \}^{R_i}$ where $D \neq (-1, -1, \ldots, -1)$ and $D \neq (1, 1, \ldots 1)$}
      \State $\iota^0_D\gets \frac{\eta_D}{2} f(s, d(s + \delta D))$
      \EndFor
      \State $\iota_0\gets \eta_{(1, 1, \ldots, 1)} f(s, d(s + \delta (1, 1, \ldots, 1))) + \sum_D \iota^0_D$\Comment{Starting point contribution}
      \State $\iota_0 \gets \iota_0 + \frac{1}{R^0} (f(s, d(d)) - f(s + \delta(1, 1, \ldots, 1), d(s + \delta(1, 1, \ldots, 1))))$
      \Else
      \State $\iota_0\gets 0$
      \EndIf
      \State $R^1\gets | \{1 \leq j \leq n : e \in D_j \} |$\Comment{Radix at ending point}
      \If {$R^1 > 1$}
      \For{$D \in \{ -1, 1 \}^{R_i}$ where $D \neq (-1, -1, \ldots, -1)$ and $D \neq (1, 1, \ldots 1)$}
      \State $\iota^1_D\gets \frac{\eta_D}{2} f(s, d(e + \delta D))$
      \EndFor
      \State $\iota_1\gets \eta_{-(1, -1, \ldots, -1)} f(s, d(e + \delta (-1, -1, \ldots, -1))) + \sum_D \iota^1_D$\Comment{Ending point contribution}
      \State $\iota_1 \gets \iota_1 - \frac{1}{R^1} (f(e, d(e)) - f(e + \delta(-1, -1, \ldots, -1), d(s +  \delta(-1, -1, \ldots, -1))))$
      \Else
      \State $\iota_1 \gets 0$
      \EndIf
      \State $\Xi\gets \iota_0 + \iota_1 + \zeta + \int_{\alpha = 0}^1 \nabla f((1 - \alpha) s + \alpha e, d((1 - \alpha) s + \alpha e)) d \alpha$
      \State {\bf {return}} $\Xi$
      \EndProcedure
    \end{algorithmic}
  \end{algorithm}
\end{section}

\begin{section}{Experiments}\label{SectionExamples}
  We report results on various datasets and model systems, and compare GIG's output with other credit allocation systems.

  \begin{subsection}{Sensitivity of credit assignment under GIG}
    In this section, we consider the behavior of GIG on models trained with systematically modified versions of the moons dataset (see Figure \ref{moons_classifier}) designed to verify that GIG can recover the appropriate information from models constructed in a particular way.

    \begin{figure}
     \centering
     \includegraphics[width=1.0\textwidth]{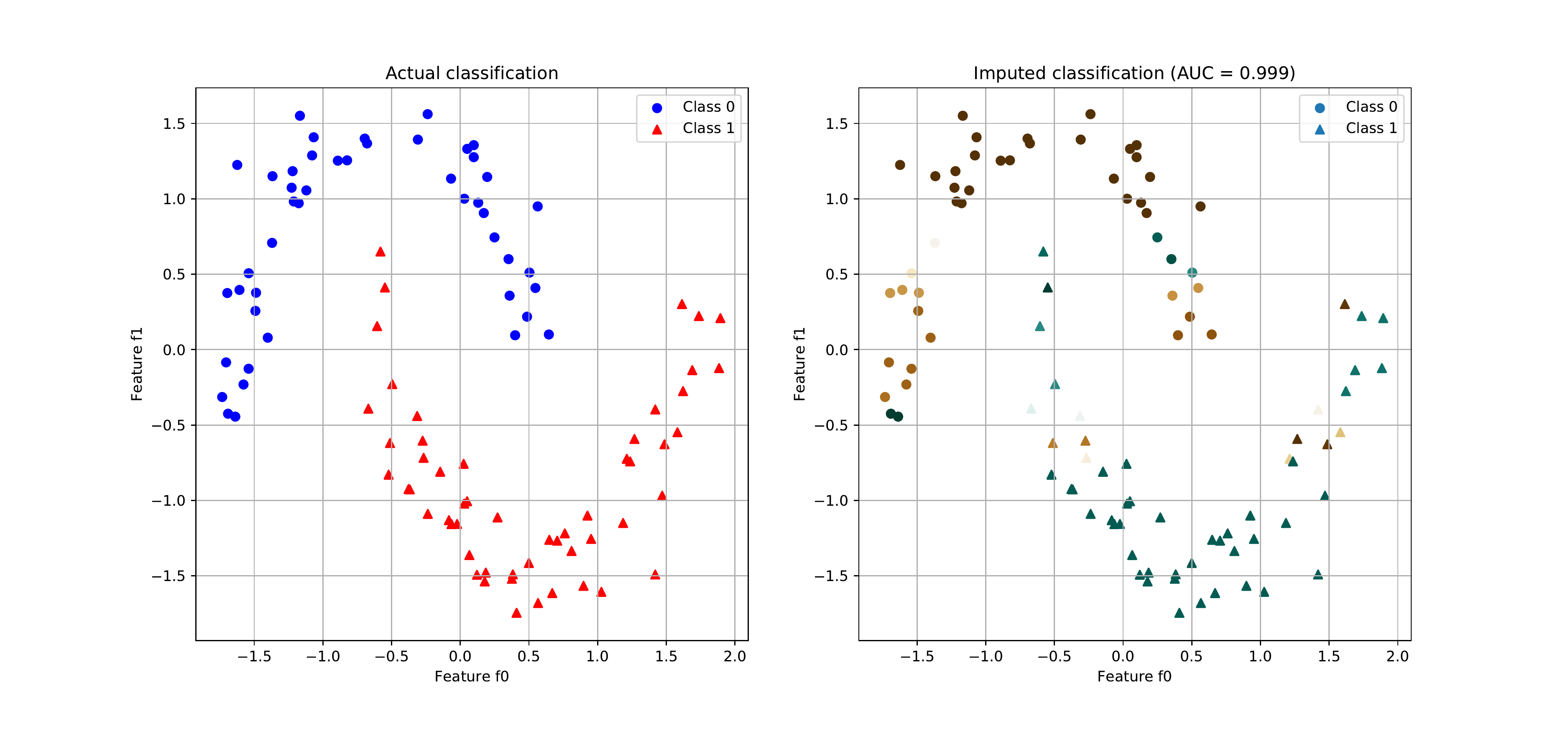}
      \begin{caption}
        {The original moons dataset and the scores arising from the XGBoost-based classifier on the members of that dataset. The scores are displayed in a color palette which is more saturated away from 0 and less saturated near zero.} \label{moons_classifier}
      \end{caption}
    \end{figure}




    The moons dataset is a two-dimensional dataset consisting of two moon-shaped targets which are perturbed by a scalable Gaussian kernel. We consider the particular version of the dataset presented by the Scikit Learn package version 0.20.3 \citep{scikit-learn}, 20,000 elements of which were extracted with the standard {\tt load\_data} call. These 20,000 elements were then normalized using the standard sklearn {\tt normalize} call. These split into a training set of 14,000 samples and a test set of 6000 samples randomly using the sklearn {\tt train\_test\_split} routine with random seed set to 0.      
    The data from the moons dataset was systematically extended with a third nuisance feature. This feature was a mixture of the target value and random noise, and three additional datasets were created, one in which the nuisance feature was pure noise, another where the nuisance feature was an equal mixture of noise and the target, and the other where the nuisance feature was the target.  For this example, three XGBoost classifiers \citep{DBLP:journals/corr/ChenG16} were built using 25 estimators of depth 6.

    \begin{figure}
      \centering \includegraphics[width=1.0\textwidth, trim=100 50 100 0,clip]{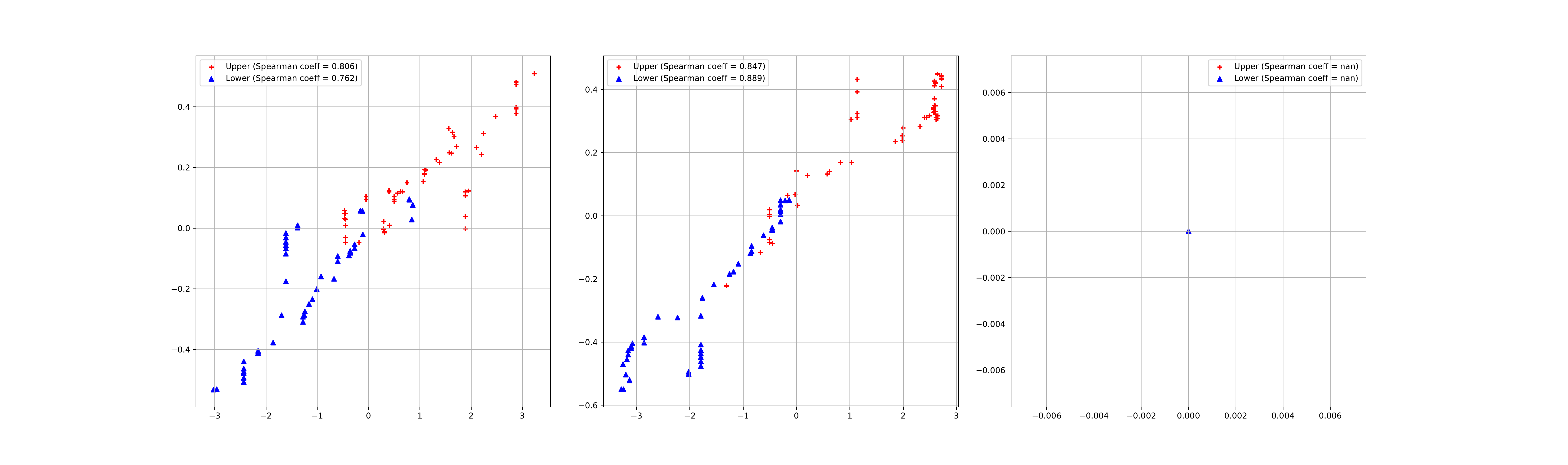} 

      \begin{caption}
        {Credit assignment with a purely random nuisance feature. The leftmost panel shows how credit is assigned to the $x$ coordinate in the moons dataset.  Each dot is the importance (vertical axis) of a given value of $x$ (horizontal axis) in determining the classification (blue or red).  The middle panel is the same for the $y$ coordinate in the dataset.  Both $x$ and $y$ help predict the outcome.  But the random nuisance feature does not, as the rightmost panel shows.  }\label{sensitivity_analysis_1}
      \end{caption} 
      
    \end{figure}

    \begin{figure}
      \centering \includegraphics[width=1.0\textwidth, trim=100 50 100 0,clip]{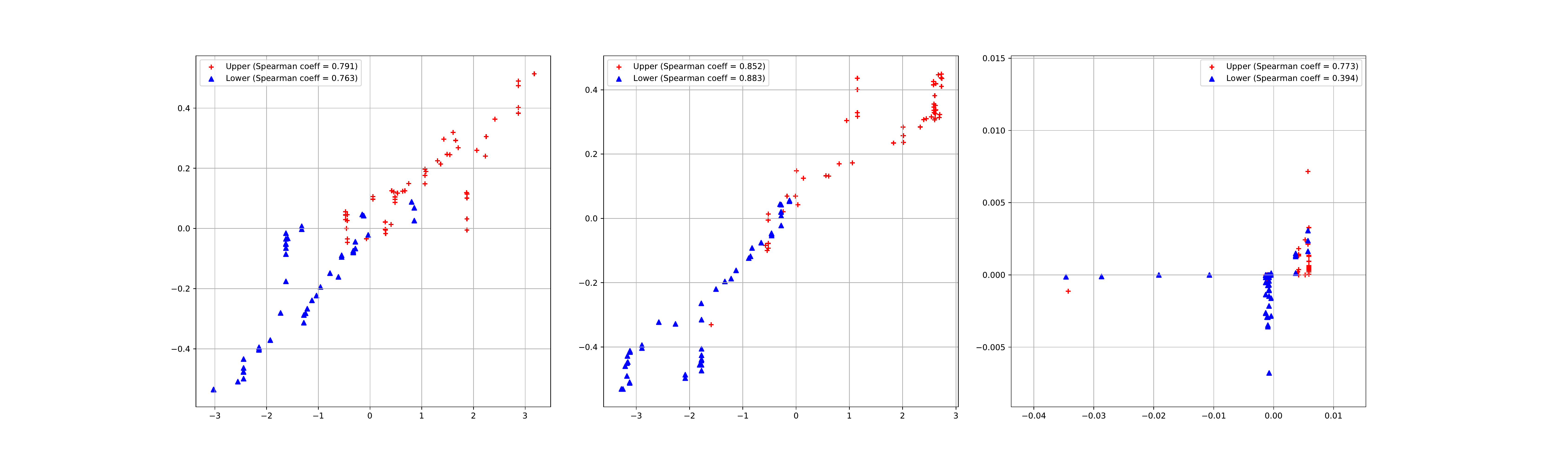} 

      \begin{caption}
        {Credit assignment with a nuisance feature that contains a mixture of a Gaussian and the target. Credit is assigned partially to all features -- including the nuisance feature -- because it has enough information to be informative, but not enough to be completely informative.}\label{sensitivity_analysis_2}
      \end{caption} 
      
    \end{figure}

    \begin{figure}
      \centering \includegraphics[width=1.0\textwidth, trim=100 50 100 0,clip]{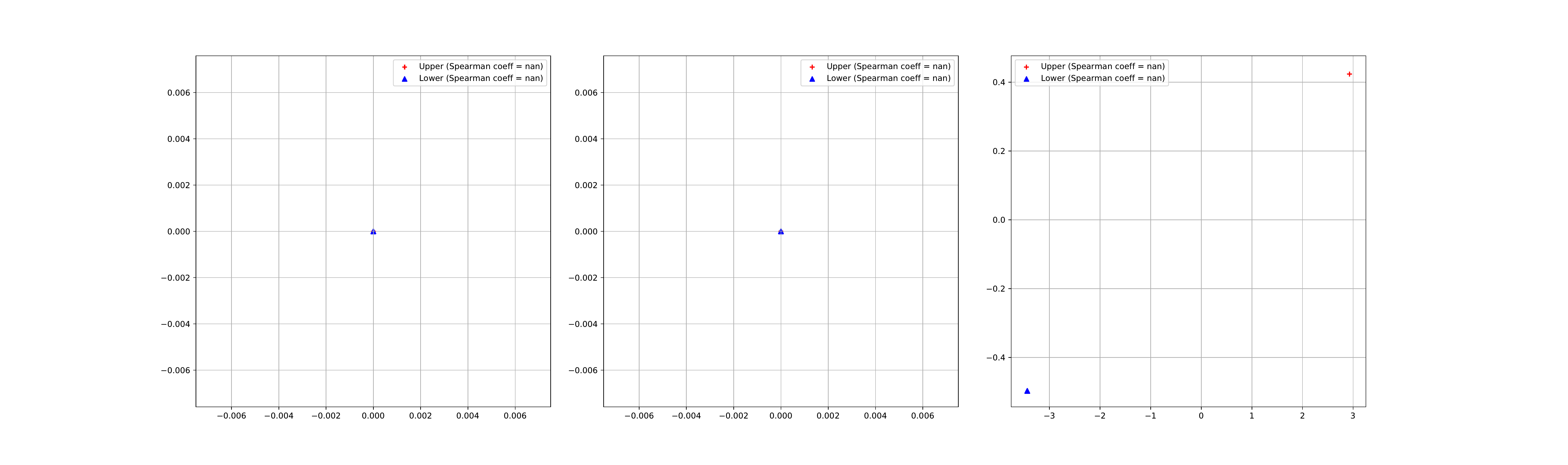} 

      \begin{caption}
        {Credit assignment when the nuisance feature equals the target. Credit is assigned only to the nuisance feature, as expected.}\label{sensitivity_analysis_3}
      \end{caption}  
      
    \end{figure}

    The models all produced high AUCs.  A reasonable credit assignment algorithm should therefore assign none of the credit to the nuisance feature in the case where it is pure noise, some of the credit to the nuisance feature when it is partially noise, and all of the credit to the nuisance feature in the case where it is equal to the target.   Indeed, as shown in Figures \ref{sensitivity_analysis_1}, \ref{sensitivity_analysis_2} and \ref{sensitivity_analysis_3}, the credit assigned by GIG for the case in which the nuisance feature was pure noise, a mixture of signal and noise, and pure signal, is as expected.

  \end{subsection}

  \begin{subsection}{Comparing GIG and TreeExplainer on the Ovals dataset}
    In this section, we consider the behavior of GIG on a toy dataset which consists of points drawn randomly according to a uniform distribution from each of two overlapping ovals. The classification task is to predict which oval a given point is associated with.  The task is intentionally impossible in the overlapping region: any point in that region could be drawn from either oval with exactly probability $\frac{1}{2}$.  As above, we use XGBoost to construct a tree-based classifier, this time consisting of 50 tree classifiers of depth 6.

    \begin{figure}
      \centering
      \includegraphics[width=0.9\textwidth]{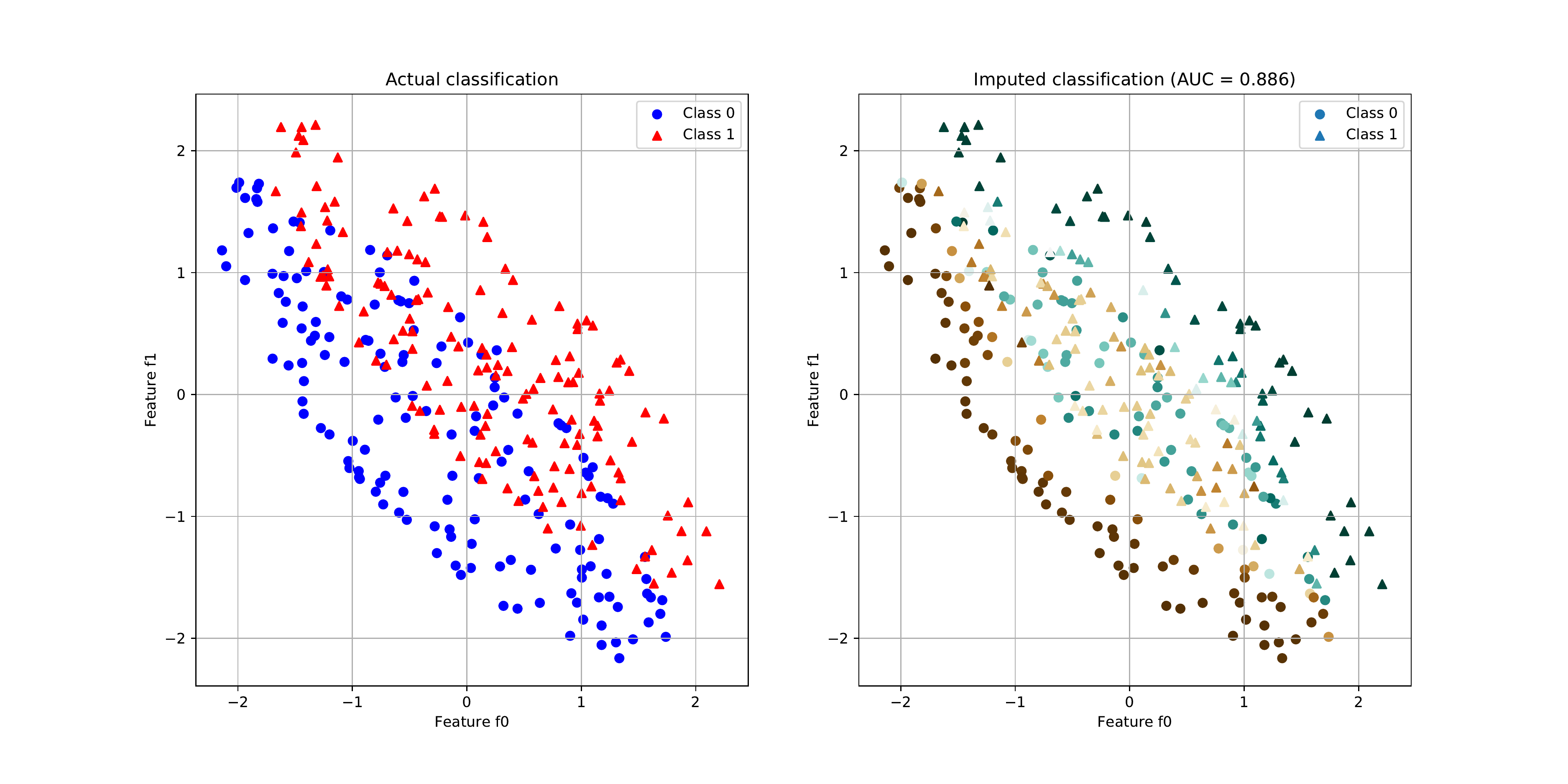}
      \begin{caption}
        {The base ovals dataset used in this demonstration. On the left, the actual dataset; on the right, the output values of the classifier on the dataset.}\label{ovals_classifier}
      \end{caption} 
    \end{figure}

    \begin{figure}
      \centering
      \includegraphics[width=1.0\textwidth,trim=100 250 100 20,clip]{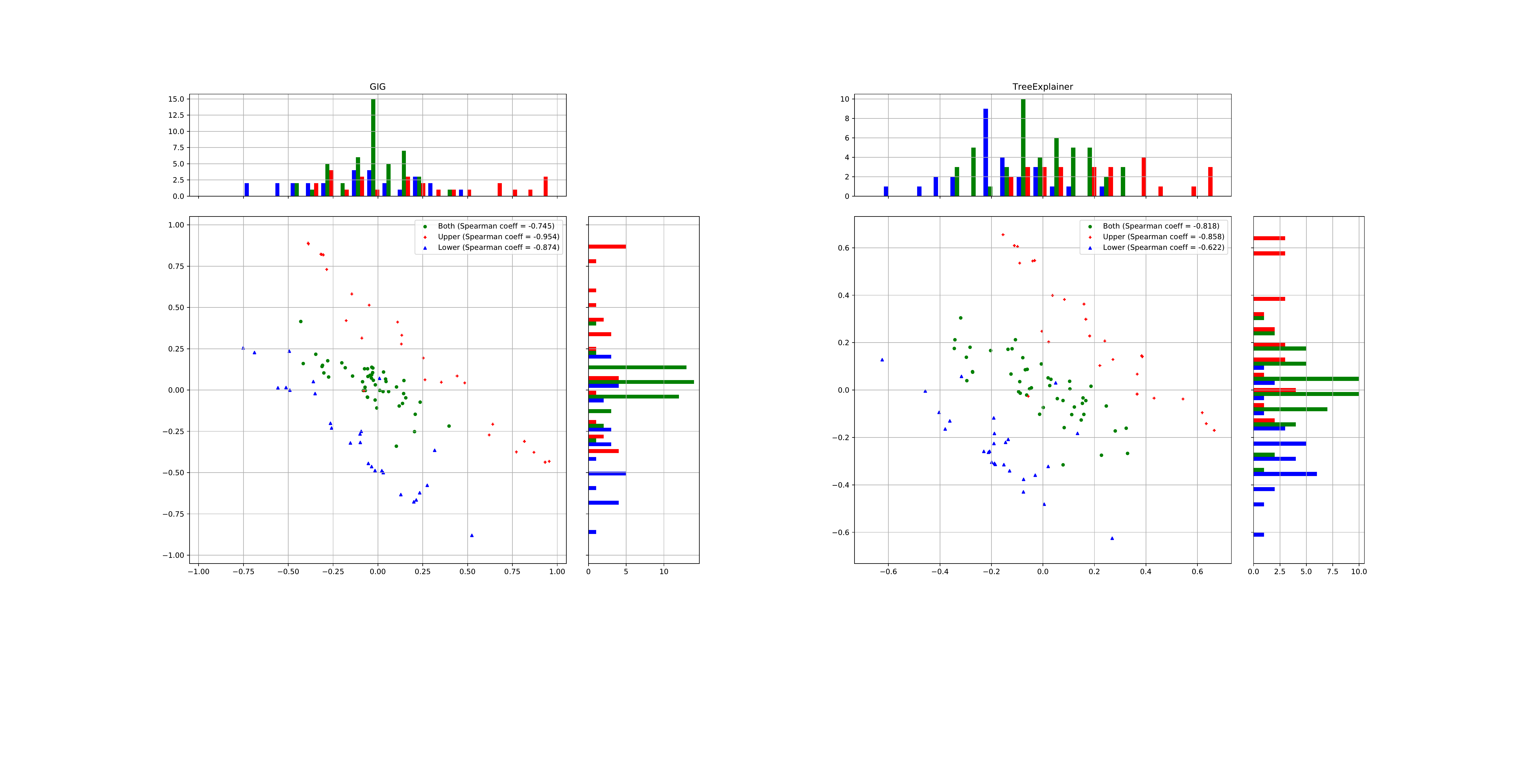}
      \begin{caption}
        {Comparison of credit allocation via GIG and TreeExplainer. The left panel shows credit allocated by GIG.  The right panel shows credit allocated by TreeExplainer.  The $x$ axes correspond to feature f0, the $y$ axes feature f1.  The red dots are from the upper oval, blue from the lower, and green from the overlapping region.  The histograms show the distributions of credit assignments for each variable.  Notice the GIG credit assignments reflect the structure of the overlapping ovals in the dataset better than those provided by TreeExplainer. The GIG labels are closer to the three lines -0.5, 0, and 0.5 for the lower, shared, and upper segments of the dataset, as shown by the Spearman coefficients for each item.}\label{SHAP_Ovals}
      \end{caption} 
      
    \end{figure}

    Figure \ref{ovals_classifier} shows, as expected, that the trained classification function is near zero in the area of the overlap, and increases as points move away from that region. Figure \ref{SHAP_Ovals} shows an interesting partition of the output values: GIG actually splits the credit attributions into three disjoint regions, one corresponding to the top group which lies outside the common area, one corresponding to the bottom group which lies outside the common area, and a third one corresponding to points which lie in the common area. In this it captures the structure of the underlying scoring manifold in ways that SHAP does not.
      \end{subsection}

  \begin{subsection}{Credit assignment for complex model types on real data set}\label{TreeCov_data}
    The covertype dataset \citep{covertype} is publicly available from the OpenML repository. The dataset was partitioned into four pieces: a training set, a training era validation set, an ensembling era validation set, and a testing set.  We trained two direct classifiers on the training set while validating on the training era validation set, and trained four ensemble classifiers on the training era validation set while validating on the validation era dataset.  We tested all these models on the common testing set.

    Two direct classifiers were trained on the dataset.  One was trained using the SciKit Learn \citep{scikit-learn} implementation of extremely random forests with 100 trees of maximum depth 10 with an initial random seed of 0 where each leaf was required to have at least two elements.  The other is a multi-layer perceptron built using Keras \citep{chollet2015keras} on top of TensorFlow \citep{tensorflow2015-whitepaper} with 55 inputs, two fully-connected ReLU layers of 1000 nodes each, a single 1000 unit layer with the TANH output function, and a single output node with a sigmoid output function.  That model was trained using the Adam optimizer with Nesterov momentum with a batch size of 100 and a learning rate of 0.01.

    The scores for each of the trained models were then rescaled from the margin space of each model into an approximate uniform output space using an piecewise linear approximation of the ECDF, thus yielding a total of eight direct models. Four ensemble models were then trained using the outputs of the ETF and MLP models as inputs. Each of these was trained using Keras over TensorFlow: two linear models, one  using the unnormalized forms of the ETF and MLP models as input, and the other  using the normalized ETF and MLP models as input, and two MLP models, each one with a single 1000 node hidden layer. All MLP models were trained with an L2 penalty of 0.001. No attempt was made to optimize the final performance, as the purpose of this was to demonstrate credit assignment.

    In Tables, \ref{TableMlpEtf}, \ref{TableMlpNormalizedMlp}, we display the values of the credit assigned to the input variables for each of the four direct classifiers. Observe that the order and amounts of the different assignments vary from one classifier to another, and that the order and amount of each assignment assignment varies between each classifier and its normalized form. In Table \ref{TableLinearEnsembleMlpEnsemble} we display the corresponding values for the ensembled classifiers. Notice also that these values are both different from one another and from those assigned by the original direct models and from their respectively normalized or unnormalized analogues.
    The results demonstrate that rank ordering and magnitude of of credit assignment changes when models are transformed, and that GIG can handle a diversity of model types, including models that are compositions of piecewise constant and continuous functions.

    \begin{figure}
      \centering
      \includegraphics[width=0.8\textwidth]{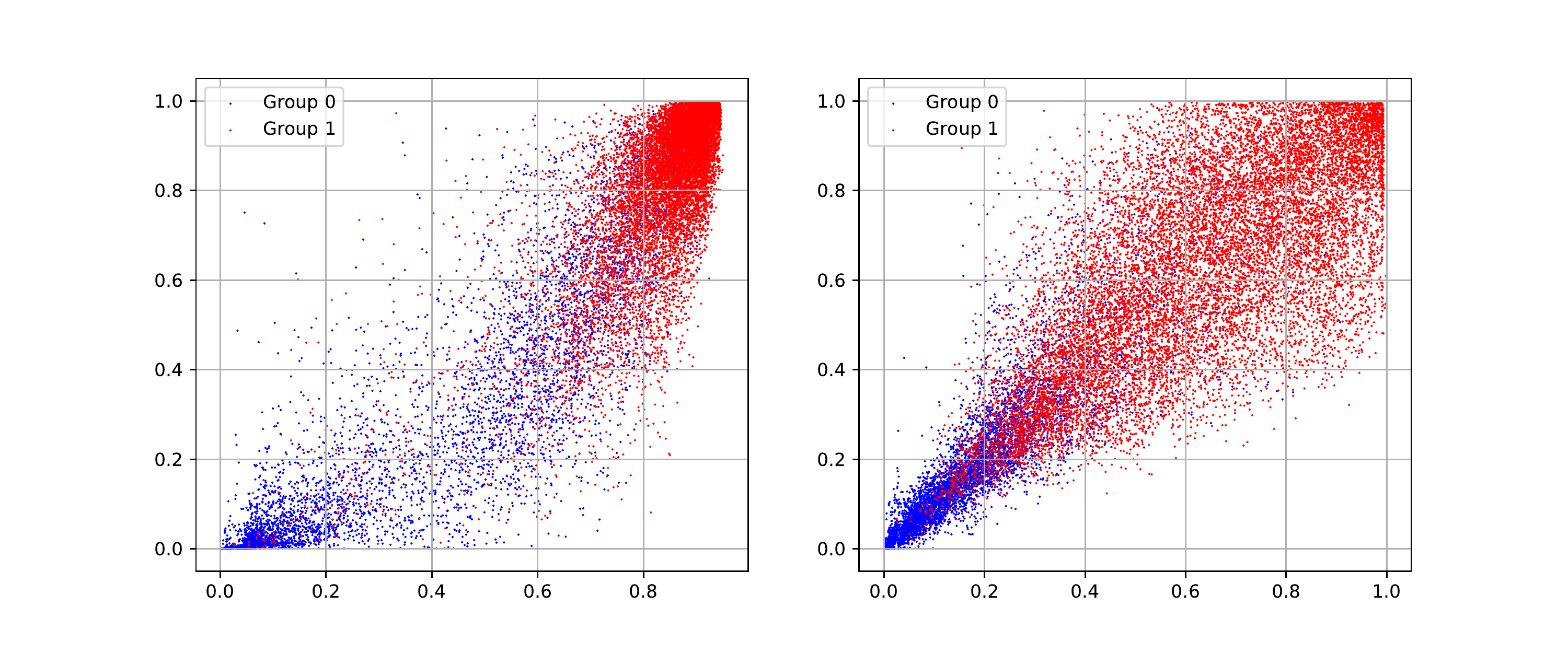}
      \begin{caption}
        {Scatter plots of the MLP and ETF classification results for the covering type classifier. The left-hand plot shows the scatter plot of the scores in margin space. The right-hand plot shows the scatter plots for the score in rank space.}
      \end{caption}\label{TreeCovScatter}
    \end{figure}


    \begin{table}[ht]\centering 
      \begin{tabular}{|l | l |l | l |}\hline
        \multicolumn{2}{|c|}{MLP} & \multicolumn{2}{c|}{ETF} \\ \hline
Horizontal Distance To Fire Points & -1.0 & Elevation & 1.0\\ \hline
Soil Type12 & 0.246 & Wilderness Area1 & -0.478\\ \hline
Aspect & 0.234 & Soil Type12 & 0.386\\ \hline
Horizontal Distance To Roadways & 0.232 & Horizontal Distance To Roadways & -0.383\\ \hline
Soil Type29 & 0.154 & Soil Type32 & 0.122\\ \hline
Slope & -0.146 & Horizontal Distance To Hydrology & 0.118\\ \hline
Horizontal Distance To Hydrology & 0.137 & Soil Type30 & 0.08\\ \hline
Hillshade 9am & 0.113 & Soil Type20 & 0.058\\ \hline
Soil Type30 & 0.08 & Soil Type16 & -0.055\\ \hline
Vertical Distance To Hydrology & 0.071 & Hillshade Noon & 0.047\\ \hline
Wilderness Area1 & -0.062 & Aspect & 0.045\\ \hline
Soil Type23 & -0.035 & Soil Type29 & 0.039\\ \hline
Hillshade Noon & 0.033 & Soil Type23 & 0.037\\ \hline
Hillshade 3pm & 0.032 & Soil Type33 & -0.036\\ \hline
Wilderness Area3 & 0.026 & Soil Type39 & 0.032\\ \hline
      \end{tabular}\caption{Credit assigned for a multi-layer perceptron and an extremely randomized trees model, which demonstrates GIG can process both a continuous and a discrete model.}\label{TableMlpEtf}
    \end{table}
    \begin{table}[ht] \centering 
      \begin{tabular}{|l | l |l | l |}\hline
        \multicolumn{2}{|c|}{MLP} & \multicolumn{2}{c|}{Smoothed ECDF(MLP)} \\ \hline
Horizontal Distance To Fire Points & 1.0 & Horizontal Distance To Fire Points & 1.0\\ \hline
Elevation & 0.755 & Elevation & 0.947\\ \hline
Horizontal Distance To Roadways & 0.562 & Horizontal Distance To Roadways & 0.616\\ \hline
Hillshade 9am & 0.35 & Hillshade 9am & 0.415\\ \hline
Horizontal Distance To Hydrology & 0.338 & Vertical Distance To Hydrology & 0.398\\ \hline
Hillshade 3pm & 0.254 & Horizontal Distance To Hydrology & 0.373\\ \hline
Aspect & 0.247 & Hillshade 3pm & 0.308\\ \hline
Hillshade Noon & 0.193 & Aspect & 0.281\\ \hline
Vertical Distance To Hydrology & 0.19 & Hillshade Noon & 0.219\\ \hline
Slope & 0.186 & Soil Type12 & 0.198\\ \hline
Soil Type29 & 0.183 & Slope & 0.168\\ \hline
Soil Type12 & 0.168 & Soil Type29 & 0.146\\ \hline
Wilderness Area1 & 0.081 & Soil Type23 & 0.06\\ \hline
Soil Type30 & 0.058 & Wilderness Area1 & 0.05\\ \hline
Soil Type23 & 0.024 & Soil Type30 & 0.043\\ \hline
      \end{tabular}\caption{Comparison of credit allocated by a multi-layer perceptron and its normalized equivalent, demonstrating the change in order and magnitude for credit assigned given the application of a smoothed ECDF.  GIG enables the accurate rank ordering of credit assignments even when the model scores are transformed.}\label{TableMlpNormalizedMlp}
    \end{table}
    \begin{table}[ht] \centering 
      \begin{tabular}{|l | l |l | l |}\hline
        \multicolumn{2}{|c|}{Normalized linear ensemble} & \multicolumn{2}{c|}{Normalized MLP ensemble} \\ \hline
Horizontal Distance To Roadways & 1.0 & Horizontal Distance To Roadways & 1.0\\ \hline
Hillshade 9am & 0.443 & Horizontal Distance To Fire Points & 0.636\\ \hline
Horizontal Distance To Fire Points & 0.428 & Hillshade 9am & 0.452\\ \hline
Elevation & 0.366 & Elevation & 0.424\\ \hline
Slope & 0.352 & Slope & 0.354\\ \hline
Horizontal Distance To Hydrology & 0.283 & Hillshade 3pm & 0.255\\ \hline
Hillshade 3pm & 0.183 & Horizontal Distance To Hydrology & 0.225\\ \hline
Hillshade Noon & 0.145 & Soil Type29 & 0.153\\ \hline
Vertical Distance To Hydrology & 0.12 & Soil Type9 & 0.15\\ \hline
Aspect & 0.098 & Aspect & 0.12\\ \hline
Soil Type29 & 0.091 & Hillshade Noon & 0.116\\ \hline
Soil Type9 & 0.077 & Vertical Distance To Hydrology & 0.113\\ \hline
Soil Type18 & 0.051 & Soil Type18 & 0.093\\ \hline
Soil Type30 & 0.022 & Soil Type23 & 0.023\\ \hline
Soil Type39 & 0.021 & Soil Type30 & 0.012\\ \hline
      \end{tabular}\caption{Comparison of credit allocated by a linear ensemble of an ETF and an MLP, passed through a smoothed ECDF (left), and an MLP ensemble of an ETF and an MLP, also passed through a smoothed ECDF.  This demonstrates that GIG can process compositions of models of mixed types, which heretofore have not been possible to accurately explain. } \label{TableLinearEnsembleMlpEnsemble} \end{table}
  \end{subsection}
\end{section}

\begin{section}{Discussion}\label{SectionDiscussion}
  \begin{subsection}{Inevitability of Equation \ref{GigEtaShapley}}\label{SubsectionShapleyRelationship}
    In a certain sense, the careful construction of the values of $\eta_{(k, j)}$, above is unnecessary. We could simply have considered a lift of $f$ which was defined by
    $$
    \xi(f, x, S) = f(x - \epsilon \chi^*(S))
    $$
    where $\epsilon > 0$ and
    $$
    \chi^*(S) = \begin{cases}
      1 & n \in S \\
      -1 & {\mathrm {o.w.}} \\
    \end{cases}
    $$
    The limits as $\epsilon$ goes to zero of the resulting terms in the computation of the Shapley values are exactly the same as the corresponding terms as computed in Equation \ref{GigEtaShapley}. That's inevitable, of course, since the Shapley values are the unique set of weights consistent with Shapley's Axioms.  Setting aside the extra axioms required to prove GIG is well-defined, there's nothing surprising about those values.

    Since this is the case, then what is the contribution of GIG to the problem of credit allocation? After all, it appears that GIG is nothing more than a clever way to merge the Aumann-Shapley values corresponding to a pair of endpoints with the Shapley values which occur at any boundary intersection.  But there is more to it than that.

    First, the computation of the Shapley values is exponential in the number of variables in the scoring function, which in the case of underwriting problems, is usually in the hundreds or thousands, and is almost always more than a handful of tens.  This makes Shapley values impractical to compute.  By contrast, the number of variables involved in any given corner in a GIG computation is rarely very large. In our studies, we've typically seen no more than a few boundary intersections in a thousand with radix $> 1$.  As such, GIG is more practical to compute.

    Second, GIG defines a unique lift that is fully determined by the choice of axioms and $f$.  Unlike with Shapley, there are no arbitrary choices to be made.  In GIG, we select a specific lift $f$ -- $\xi$, defined entirely from the values of $f$ in the various orthants. We've shown above that the selection of a lift in the case of Shapley is arbitrary.  By contrast, GIG is unique: given a function and a pair of points to be compared, there is only one corresponding allocation that satisfies GIG's axioms.  This means credit allocations computed by GIG can be viewed as objective, not subjective, as they are with Shapley and its derivatives.
  \end{subsection}

  \begin{subsection}{Comparison with SHAP}\label{SubsectionShap}
    In \citet{lundberg2017unified,lundberg2018consistent,Shap2019:463e71ac98637be58d8fce2a16ed42ed3cdbf0cc}, Lundberg and colleagues describe methods that allocate credit for several well-known ML algorithms in terms of the Shapley values.
    
    As discussed above in \ref{SubsectionTransformation}, many practical uses of machine learning systems require explanations in a transformed output space.  This proves challenging for many algorithms.  In addition, the algorithms described in \citet{lundberg2017unified,lundberg2018consistent,Shap2019:463e71ac98637be58d8fce2a16ed42ed3cdbf0cc} depend on the assumption of feature independence (as in Equation 11 in \citet{lundberg2017unified}).

    Requiring feature independence is problematic in many applications, where features are correlated based on how they are constructed (e.g., a credit score is usually computed based on other model variables that are also included in the model) or by virtue of their encoding (e.g., one-hot encoded categoricals are obviously not independent).

    Prior algorithms either:  (1) cannot handle the score-space, margin-space transformation or (2) require the assumption that features  are independent.  As a result, the results from applying these methods are less accurate than GIG, which has none of these limitations.  GIG is well-defined regardless of whether the variables in the model are statistically independent.


    DeepSHAP \citep{lundberg2017unified} is a technique for providing explanations for neural networks \citep{backprop-article}.  DeepSHAP computes neural network importance no matter what the final remapping is, whether the results are reported in margin space or in some transformed score space.  However, DeepSHAP is limited to explaining neural network models and, because it uses Shapley sampling to approximate the Shapley values, it depends on a dubious feature independence assumption.

    TreeExplainer \citep{lundberg2018consistent} applies to ensembles of decision trees. TreeExplainer is very fast and completely general if the result of the model is reported in margin space. However, TreeExplainer only reflects the correct computations of the Shapley values when the result is otherwise transformed, if we can assume that the input features are independent (which, as we discussed above is almost never true).

    Given a pure linear ensemble of one or more neural networks and one or more tree classifiers, we can extend the pair of DeepExplainer and TreeExplainer to create an ``EnsembleExplainer'' -- provided everything is ultimately in margin space. If we were to transform the output of the ensemble into some other space, we would again immediately require signal-wise independence.

    And, worse, if we take any fundamentally non-linear ensemble of such a set of items, we would no longer be able to use ``EnsembleExplainer'' at all. In \citet{lundberg2017unified}, however, the authors present KernelExplainer, a mechanism for computing the Shapley values of any arbitrary function at any arbitrary point -- provided the input features are independent, again a disqualifying assumption for our applications.

    Finally, none of these explainers or published analysis provide any reason to believe that the particular lift they employ is the correct lift.  
  \end{subsection}
  \begin{subsection}{GIG is fully determined by its axioms}\label{SubsectionGeometry}
    When one computes the Shapley values for a particular atomic game, one relies on a set function from the power set of the set of all input features to the set of output values. That's an essentially combinatorial operation: the function which we're actually assigning credit for isn't the original function, but rather a combinatorial lift of that function into a higher-dimensional space. This lift is essentially arbitrary: Lundberg et al. use a simple column-by-column lift when dealing with tabular data in the patient diagnosis application they discuss, but use super-pixel data in the cases where they are explaining image discrimination. These methods pick and processes {\em subsets} not {\em values}, and they do this in an arbitrary, but problem-dependent fashion.

    By contrast, GIG is fully determined by its axioms. By the proof of Theorem \ref{GigTheorem}, there is one and only one choice for the terms of the computation. Although the coefficients in Equation \ref{GigEtaShapley} are formulaically similar to the coefficients used in the computation of the Shapley values, their actual meaning is quite different: there's no arbitrary combinatorial lift into a higher dimensional space, but rather a computation based on orthants around a given intersection. Instead, GIG makes assuptions about the locally smooth structure of the manifold of possible input values.

    One of the interesting consequences of GIG's fully determined geometric or topological structure is a measure-theoretic interpretation of it as a direct extension of IG. IG depends on the notion of a path integral with respect to standard Lebesgue measure between two points. GIG constructs a notion of a ``path integral'' with respect to the Lebesgue decomposition of the measure induced by the combination of Lebesgue measure on the interval between the two endpoints and the orthogonal Borel measure induced by the hyperplanes upon which the function is discontinuous. We showed above that the sum of the values of $\iota$ and $\zeta$ create a well-defined notion of path integral in this case.  We believe GIG is the only reflexive linear operator on the space of all paths and functions continuous with respect to such a measure.
  \end{subsection}
\end{section}
 
\begin{section}{Summary}
  We described GIG, a novel credit allocation algorithm for a broad class of scoring functions that are of practical importance in real-world applications where accurate explanations are required.  Unlike other approaches, GIG's credit allocation is fully determined by its axioms and the scoring function under study.  We showed GIG is unique, practical to compute, and delivers expected results on a variety of datasets and ensembles of models.  

\end{section}

\acks{The authors gratefully acknowledge helpful consultations with
  Scott Lungberg and Manuela Veloso.}
\bibliography{citations.bib}

\end{document}